\def\genericversion{1}
\documentclass{article}
\ifdefined\genericversion
\usepackage[margin=1in]{geometry}
\usepackage{times,natbib}
\else
\usepackage{iclr2027_conference,times}
\fi
\usepackage[T1]{fontenc}
\usepackage{amsmath,amssymb,amsthm,mathtools,bm,booktabs,array,multirow,enumitem,graphicx,hyperref,url,xcolor,tikz,float}
\usepackage[ruled,vlined,linesnumbered]{algorithm2e}
\usetikzlibrary{arrows.meta,positioning,fit,shapes.geometric,backgrounds}
\hypersetup{colorlinks=true,hypertexnames=false,linkcolor=blue!50!black,citecolor=blue!50!black,urlcolor=blue!50!black}
\setcitestyle{numbers,square}
\definecolor{BLblue}{HTML}{173F5F}
\definecolor{BLorange}{HTML}{D97706}
\newtheorem{definition}{Definition}[section]
\newtheorem{theorem}[definition]{Theorem}
\newtheorem{corollary}[definition]{Corollary}

\newcommand{\E}{\mathbb E}
\newcommand{\pa}{\operatorname{pa}}

\newcommand{\stageletter}[1]{{\color{BLblue}\fontsize{6}{5.5}\selectfont\bfseries #1}}

\title{Finding Icebergs in Language-Model\\Workflows\\[0.3em]\large Diagnosing Latent Structural Fragility with Stochastic\\Semantic Evidence Graphs}
\ifdefined\genericversion
\author{%
Matthew F. Dixon$^{1,2}$ \qquad Bertrand Nortier$^{3}$ \qquad Miquel Noguer i Alonso$^{2,4}$\\[0.5em]
\small $^{1}$BeliefLens \quad $^{2}$AI Finance Institute (AIFI) \quad
$^{3}$University of Leicester \quad $^{4}$Cornell University\\[0.25em]
\small \texttt{matthew.dixon@belieflens.org}}
\date{September 2026}
\else
\author{Anonymous authors}
\date{}
\fi

\begin{document}
\ifdefined\internalcover
\begin{titlepage}
\thispagestyle{empty}
{\color{BLblue}\LARGE\bfseries Internal empirical-completion cover sheet\par}
\vspace{0.4em}
{\large Finding Icebergs in Language-Model Workflows\par}
\vspace{0.35em}
{\small Updated 20 September 2026.  This page is a working checklist and is not part of the submission.  No result enters the manuscript until its code, population, split, labels, intervention construction and statistical analysis pass an independent audit.\par}

\vspace{0.8em}
{\color{BLblue}\large\bfseries Current audit disposition\par}
\begin{itemize}[leftmargin=*,nosep]
\item \textbf{Retained headline:} MAGIC rank-1 localization is independently reproducible on 960 untouched cases.  The paper now describes last-conflict rank as a retrospective oracle-stopping diagnostic, not a deployable stopping rule or a measure of evidence volume.
\item \textbf{Corrected appendix result:} the graph-misspecification figure and discussion now use the interval-action v2 analysis.  Under combined stress the corrected robust gate takes no automatic action; the discrepancy-bound coverage and zero-false-pass results remain valid.
\item \textbf{ALCE and factuality baselines complete:} five common-control model outputs pass audit; Mistral remains withheld pending its corrected duplicated-BOS rerun.  Official RAGAS 0.4.3 Faithfulness has now been run on all 1,467 held-out links, and the matched SSEG--RAGAS analysis uses identical labels and question groups.
\item \textbf{Self-RAG comparison complete:} a frozen three/four-hop MuSiQue study compares terminal Self-RAG reflection scores with the same measurements plus retained path structure.  All 1,098 generated conditions pass the structural audit; on 560 untouched source interventions, SSEG strongly improves detection of response-destabilizing evidence changes.  The prespecified absolute-score localizer is null; a clearly labelled post-hoc pilot shows that measuring paired local edge contrasts recovers exact source attribution, pending independent confirmation.
\end{itemize}

\vspace{0.65em}
{\color{BLblue}\large\bfseries Three outstanding empirical items\par}
\begin{enumerate}[leftmargin=*,itemsep=0.35em,topsep=0.25em]
\small
\item \textbf{Corrected cross-model CONFLICTS.} Complete and independently audit the full-sequence Qwen and OLMo runs now in progress.  Admit them only if prefix alignment, expressed mass, final-answer extraction and aggregation pass the frozen sentinels.
\item \textbf{Sixth ALCE model.} Replace the excluded Mistral output with the corrected duplicated-BOS rerun under the identical common-control map.  The current five-model result remains the submission claim until this audit passes.
\item \textbf{Official FActScore.} Run the released atomic-fact pipeline, or state explicitly why its open-domain knowledge-source target is not an equal-population claim--source comparator.  The existing DeBERTa result remains labelled FActScore-style.
\end{enumerate}

\vspace{0.35em}
{\color{BLblue}\large\bfseries Completed since the programme-chair review\par}
{\small Official RAGAS 0.4.3 Faithfulness is evaluated on all 1,467 held-out ALCE links; the duplicate-safe HaluEval study covers 8,050 pairs; and the frozen ToolSandbox limitation study covers 96 executions across two agent models.  The offline reproducer now reconstructs every current headline result from case-level records or frozen conditional laws without provider or network calls.  An audited workload profile reports every available condition, call and timing quantity while leaving unrecorded common-hardware latency and monetary cost explicitly unavailable.\par}

\vspace{0.45em}
{\small\textbf{Final integration rule.}  Prefer one evidence ladder---conditional bound validation, independently labelled path localization and externally labelled qualification---over additional applications.  Negative or inconclusive experiments remain disclosed as limitations rather than being optimized away.\par}
\end{titlepage}
\setcounter{page}{1}
\fi
\ifdefined\supplementonly
\else
\maketitle

\begin{abstract}
AI-workflow governance cannot be reduced to checking the final answer: an apparently safe answer may rest on a fragile evidence path that ordinary evaluation cannot see, localize or govern.  We call this hidden fragility a \emph{structural iceberg}: hallucinations and unsupported claims may form its visible tip, while consequential weakness remains submerged.  Stochastic semantic evidence graphs (SSEGs) expose these icebergs by preserving workflow channels, propagating local uncertainty and identifying the hidden paths on which an apparently safe output depends.  ALCE and RAGTruth show that visible failures at the tip---unsupported citations and hallucinated spans---rest on distinct submerged weaknesses and therefore require different interventions.  Across retrieval, tool-use and controlled stress tests, SSEG localizes those weaknesses, supports targeted repair, produces no false automatic passes in 35,000 known-truth cases and reduces ToolSandbox review by 28.8\% across 96 executions from two agent models.  The same structural view carries into end-to-end governance: in a separately sealed 1,200-case FinGovBench study, adding SSEG to GPT-OSS-20B reduces unsafe releases from 452/660 to 8/660 while releasing all 540 safe cases and correctly distinguishing 592/600 matched workflow pairs.  An unchanged-gate transfer to Qwen3-8B releases all 540 safe cases and none of 660 unsafe cases, whereas flat-UQ releases 520 unsafe cases.  SSEG therefore moves governance below surface-level output checking, turning hidden evidence dependencies into auditable, path-specific decisions about intervention, revalidation and release.
\end{abstract}

\section{Introduction}

\paragraph{The practical problem: the answer is only the visible tip.}
An AI workflow may retrieve evidence, arrange a prompt, generate an answer and pass that answer to a decision rule.  A hallucination or failed check is visible at the surface.  The larger governance problem may remain hidden below it: an apparently acceptable answer can depend on weak evidence or on a workflow component whose error materially affects the final action.  We call this hidden dependency a \emph{structural iceberg}.  Ordinary output checks assess the visible tip; they do not show where a weakness entered, what depended on it or what must be revalidated.

\paragraph{Related approaches.}
Language-model uncertainty has been studied through self-evaluation \cite{kadavath2022}, semantic entropy \cite{kuhn2023}, concept-level calibration \cite{nakkiran2026}, faithful linguistic uncertainty \cite{liu2025}, and uncertainty in instruction following and function calling \cite{heo2025,ye2026}.  Predictive uncertainty can degrade under shift \cite{ovadia2019}; selective prediction and conformal risk control formalize abstention and risk--coverage \cite{geifman2019,angelopoulos2024}.  RAGAS (automated evaluation of retrieval-augmented generation) and ARES (an automated evaluation framework for retrieval-augmented-generation systems) score retrieval and generation quality \cite{ares2024,ragas2024}; FActScore decomposes text into atomic facts \cite{factscore2023}; and Self-RAG trains models to retrieve and critique \cite{selfrag2024}.  These methods estimate uncertainty, assess support or improve generation.  They do not, however, by themselves preserve the versioned workflow path that produced a discrepancy, bound its downstream effect and identify the components requiring revalidation.  That is the gap addressed here.

AI-governance research has separately emphasized lifecycle-wide internal auditing \cite{raji2020}, standardized model and service documentation \cite{mitchell2019,arnold2019}, and traceable measurement linked to risk-management action \cite{nistairmf2023}.  Provenance systems can record interactions across an agentic workflow \cite{provagent2025}.  These approaches motivate auditable records and organizational controls; SSEG adds a quantitative, path-specific account of how an observed weakness can reach a governed action.

\paragraph{What we introduce.}
A \emph{stochastic semantic evidence graph} (SSEG) is a directed graph of the observable workflow.  Its nodes retain meaning-bearing objects---for example, evidence passages, claims, evaluation results and the governed action---and its edges record which outputs were supplied to which later components.  Each node carries a local uncertainty measurement; each edge bounds how strongly a discrepancy can be transmitted.  SSEG therefore preserves both the magnitude and the location of uncertainty.  We use it through the four-stage Expose--Localize--Route--Govern (ELRG) workflow:
\begin{enumerate}[leftmargin=*,nosep,label=\textbf{\arabic*.}]
\item \textbf{Expose} a hidden dependency rather than reducing the workflow to terminal confidence.
\item \textbf{Localize}---that is, path-localize and rank---the nodes and retained paths that contribute most to uncertainty in the final action.
\item \textbf{Route} the alarm to the retained source or component that requires review.
\item \textbf{Govern} the action by comparing propagated uncertainty with a declared tolerance.
\end{enumerate}
The framework complements rather than replaces calibration, entailment and task-specific controls.  It may contribute directly to a release decision when its measurements and tolerances have been qualified, or it may explain an alarm raised by a separately qualified gate.  Its distinctive object is the versioned answer--evidence dependency.  ELRG connects lifecycle auditing and trace provenance \cite{provagent2025,raji2020} with selective and risk-controlled release \cite{geifman2019,angelopoulos2024}; the pathwise composition of these functions is the contribution developed here.

\paragraph{Central thesis.}
A terminal pass---the commonplace endpoint of many governance approaches---can mean either that the workflow is robust or that the monitor is blind to a fragile dependency.  SSEG distinguishes these possibilities only to the extent that local measurements detect the weakness.  It is not a universal factuality classifier.  It links uncertainty measurement to a concrete governance action: release, inspect a named path or revalidate its descendants.  In this paper, \emph{discrepancy} is a distance from declared reference behaviour, \emph{error} requires an independent label or outcome, and \emph{uncertainty} is represented by a probability law or identified set.

\paragraph{How to read Figure~\ref{fig:overview}.}
Figure~\ref{fig:overview} gives a conceptual overview of SSEG: it follows an AI workflow from a user request, external evidence and tool results through RAG retrieval and prompt construction, into the language model's token sequence and generated claim, and finally through evidence checking to a governed terminal action.  The inputs are the request, available evidence and provenance, and any tool result.  Nodes represent observable workflow components or, within the language model, generated tokens.  Directed edges record which component outputs are supplied downstream; they are data dependencies, not attention weights or inferred reasoning steps.  The output is a candidate action together with propagated uncertainty and a trace of the paths responsible for it, which ELRG uses to release the action or route it for review.

\begin{figure}[t]
\centering
\resizebox{.99\textwidth}{!}{%
\begin{tikzpicture}[
  vnode/.style={draw=BLblue,circle,fill=white,minimum size=9mm,inner sep=0pt,font=\small\bfseries},
  root/.style={vnode,fill=green!7},
  terminal/.style={vnode,draw=BLorange,fill=yellow!18},
  elrg/.style={draw=BLorange,rounded corners=10pt,fill=orange!5,minimum height=13mm,minimum width=114mm,align=center,font=\scriptsize,inner xsep=9pt},
  tok/.style={vnode,minimum size=7mm,font=\tiny},
  flow/.style={-{Stealth[length=1.7mm]},thick,draw=BLblue},
  uqflow/.style={-{Stealth[length=1.9mm]},very thick,draw=BLorange},
  typelab/.style={font=\scriptsize,text=black!75,align=center},
  lab/.style={font=\tiny,text=black!70,fill=white,inner sep=.7pt},
  ssegfield/.style={draw=BLblue,dashed,rounded corners=14pt,fill=blue!1,inner sep=7pt},
  lmfield/.style={draw=black!45,densely dotted,rounded corners=7pt,fill=orange!2,inner sep=4pt},
  keybox/.style={draw=black!30,rounded corners=4pt,fill=white,align=center,font=\scriptsize,inner xsep=5pt,inner ysep=3pt}
]
\node[root] (u) at (0,1.8) {$U$};
\node[root] (e) at (0,0) {$E$};
\node[root] (o) at (0,-1.8) {$O$};
\node[font=\scriptsize\bfseries,text=black!70] at (.35,2.75) {INPUTS $x=(Z_U,Z_E,Z_O)$};

\node[vnode] (r) at (2.35,.9) {$R$};
\node[vnode] (p) at (4.05,0) {$P$};
\node[tok] (t1) at (5.25,0) {$t_1$};
\node[tok] (t2) at (6.25,0) {$t_2$};
\node[font=\scriptsize] (td) at (7.05,0) {$\cdots$};
\node[tok] (tm) at (7.85,0) {$t_m$};
\node[vnode] (c) at (8.95,0) {$C$};
\node[vnode] (h) at (10.15,0) {$H$};
\node[terminal] (term) at (11.45,0) {$T$};

\node[typelab] at (0,1.25) {request};
\node[typelab] at (0,-.74) {evidence\\$+$ provenance};
\node[typelab] at (0,-2.38) {tool result};
\node[typelab] at (4.05,.62) {prompt};
\node[font=\tiny,text=BLblue,align=center] (toklabel) at (6.55,-.72)
  {conditional token law: $K_j(t_j\mid t_{<j},P)$};
\node[typelab] at (8.95,.68) {claim};

\coordinate (ssegbottom) at (5.75,-2.55);
\begin{scope}[on background layer]
\node[ssegfield,fit=(u)(e)(o)(r)(p)(t1)(t2)(tm)(c)(h)(term)(toklabel)(ssegbottom),label={[font=\small\bfseries,text=BLblue]above:SSEG}] (sseg) {};
\node[lmfield,fit=(t1)(t2)(td)(tm),label={[font=\tiny\bfseries,text=black!65]above:LANGUAGE MODEL}] (lm) {};
\end{scope}
\node[elrg] (out) at (6.35,-4.85) {\textbf{ELRG applies terminal bound $B_T^+$ to candidate action $Z_T$}\\[1mm]\textbf{Expose} $B_T^+>\tau$? $\quad\vert\quad$ \textbf{Localize} path-rank $\rho_{vT}$ $\quad\vert\quad$ \textbf{Route} artifact $+\mathcal R(v)$ $\quad\vert\quad$ \textbf{Govern} release $Z_T$/review};

\draw[flow] (u.east) -- node[lab,above] {$Z_U$} (r.north west);
\draw[uqflow] (e.east) -- node[lab,above] {$Z_E$} (r.south west);
\draw[uqflow] (r.east) -- node[lab,above] {$Z_R=\bm S$} (p.north west);
\draw[flow] (o.east) -- node[lab,below] {$Z_O$} (p.south west);
\draw[uqflow] (p.east) -- (t1.west);
\draw[uqflow] (t1.east) -- (t2.west);
\draw[uqflow] (t2.east) -- (td.west);
\draw[uqflow] (td.east) -- (tm.west);
\draw[uqflow] (tm.east) -- (c.west);
\draw[uqflow] (c.east) -- (h.west);
\draw[uqflow] (h.east) -- (term.west);
\draw[flow] (r.north) -- (2.35,1.9) -- node[lab,above] {grounding evidence + provenance} (10.15,1.9) -- (h.north);
\draw[uqflow] (term.south) -- node[lab,right,pos=.82] {$Z_T,\ B_T^+$} (term.south |- out.north);

\node[typelab,fill=blue!1,inner sep=1.2pt] at (2.35,1.48) {RAG retrieval};
\node[typelab,fill=blue!1,inner sep=1.2pt] at (10.15,.84) {evidence check};
\node[typelab,fill=blue!1,inner sep=1.2pt] at (11.45,-.68) {terminal action};

\node[keybox] at (5.35,-3.55) {\textbf{KEY}\quad edge $p\to v$: $Z_p$ supplied to $v$\\
\textcolor{BLblue}{\textbf{blue edge}}: recorded dependency \quad\textbar\quad
\textcolor{BLorange}{\textbf{orange edge}}: uncertainty propagated toward $T$\\
\textcolor{green!45!black}{\textbf{green node}}: input \quad\textbar\quad
\textcolor{yellow!65!orange}{\textbf{yellow node}}: terminal candidate};
\end{tikzpicture}}
\caption{The SSEG shown explicitly as a DAG.  The observed inputs are $x=(Z_U,Z_E,Z_O)$.  Equal-sized nodes represent observable components.  An edge $p\to v$ records that output $Z_p$ is supplied to child $v$.  Retrieval node $R$ is the RAG component: it selects context from evidence input $E$ and returns the retained source set $Z_R=\bm S=(S_1,\ldots,S_J)$, which is supplied to prompt node $P$ and preserved for evidence check $H$.  Thus $R$ denotes the retrieval operation, whereas $S_j$ denotes one retained source.  The language-model component contains the autoregressive chain $t_1\to\cdots\to t_m$, with conditional kernel $K_j(t_j\mid t_{<j},P)$.  Evidence check $H$ therefore receives both generated claim $C$ and its retained sources from $R$.  Direct source--token or tool--token edges are added only when declared traces or controlled interventions identify them; they are not inferred from attention alone.  Here $Z_T$ is the candidate terminal action, $B_T^+$ is its robust propagated uncertainty bound and $\tau$ is the declared release tolerance.  \emph{Expose} applies the terminal check $B_T^+>\tau$; \emph{Localize} path-ranks contributions $\rho_{vT}=\Gamma_{vT}b_v$; \emph{Route} maps a leading node $v$ to its retained review artifact $\mathcal R(v)$; and \emph{Govern} releases $Z_T$ or sends the affected path for human or automated review.  After review or repair, the changed node and its recorded descendants are recomputed before the gate is applied again.}
\label{fig:overview}
\end{figure}
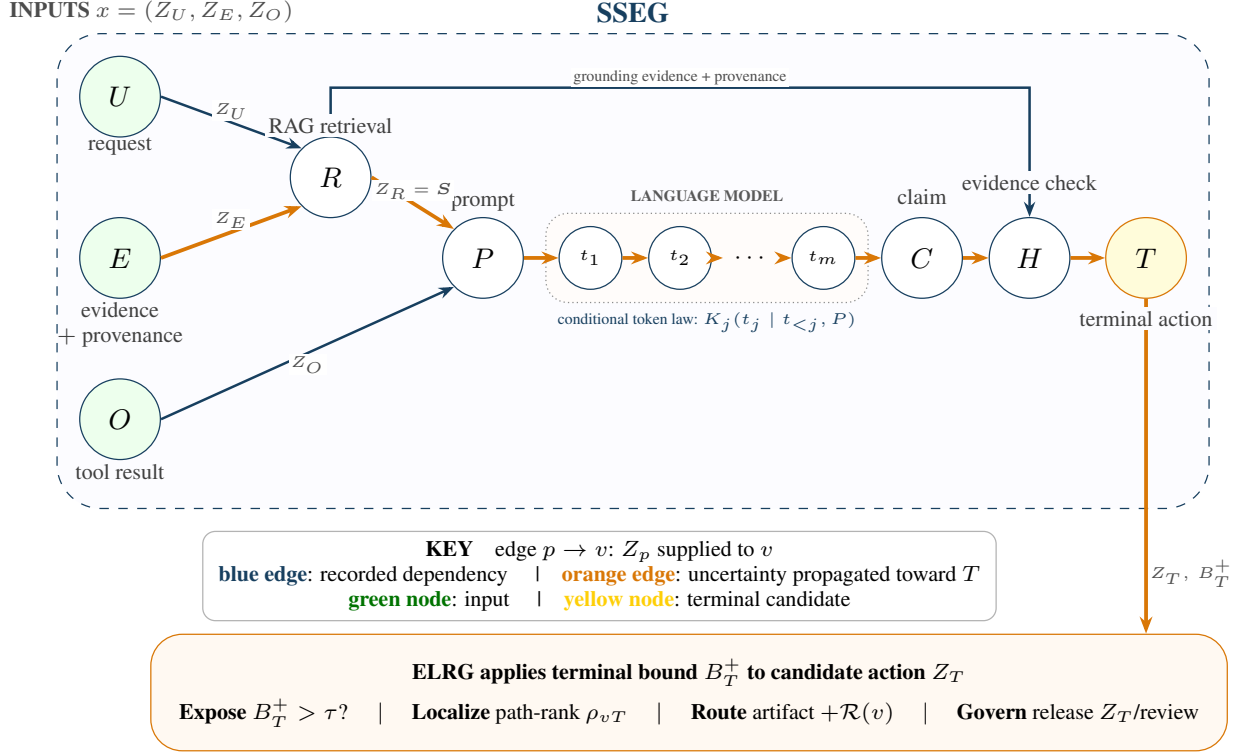

\paragraph{Overall contribution and positioning.}
The novelty of SSEG is that it connects uncertainty measurement, evidence assessment, provenance and audit within a statistical governance layer for multi-component language-model workflows.  Unlike terminal uncertainty, factual-support or retrieval-quality scores, it retains the directed workflow topology, attaches measurements to the components that produced them, propagates their possible effects to the terminal action and returns the responsible paths for review.  The theoretical contribution is a conditional qualification framework for typed stochastic workflows: it decomposes terminal discrepancy into nodewise defects and their downstream path influences, and provides a robust extension when graph structure or local bounds are uncertain.  Connecting observable workflow interventions to these local terms produces an auditable, path-specific certificate rather than a single terminal score.  The broader methodological contribution is to integrate this certificate with typed workflow traces, controlled interventions and local probabilistic measurements in the actionable ELRG governance procedure.  Empirically, the paper evaluates this construction through known-truth stress tests, independently labelled retrieval and provenance benchmarks---several evaluated across multiple model families---tool-use workflows and a sealed end-to-end financial-governance study.  This combination tests not only whether a gate detects risk, but whether it can expose a hidden dependency, path-localize the affected route and specify what must be reviewed or revalidated.

\paragraph{Overview of the paper.}
Section~2 defines the SSEG representation, its observation profiles and its treatment of token, provenance and claim--source uncertainty.  Section~3 derives the conditional pathwise bound that carries local discrepancies through the DAG to the governed action and explains how the resulting path contributions drive ELRG.  Section~4 presents the evidence ladder and experiments, moving from controlled cases with known truth to external localization benchmarks, tool-use studies and sealed end-to-end governance.  Section~5 summarizes the practical implications and limitations.  The appendices provide the joint-law construction, robust extension, proof, experimental protocols, additional limitation studies, compute record and reproducibility instructions.

\section{Stochastic semantic evidence graphs}

\paragraph{Graph-level view.}
Let $G=(V,\mathcal E_G)$ be a finite directed acyclic graph, with nodes $V$ and directed edges $\mathcal E_G$; loops are represented by unrolling them through time.  The output $Z_v$ of node $v$ takes values in an associated standard Borel space $(\mathcal Z_v,\mathcal B_v)$.  Each node also has a declared, application-specific type such as evidence, retrieval, claim, evaluation, tool or decision; this list is not exhaustive.  Write $\pa(v)$ for the parents of $v$ and $Z_{\pa(v)}=(Z_p:p\in\pa(v))$ for their outputs.

\paragraph{What is held fixed.}
All comparisons condition on a versioned workflow configuration $\lambda$, which records properties of the retriever, prompt and evaluator, together with the governance loss and permitted actions.  These choices are held fixed when workflows are compared.  In the main experiments, ``tool'' refers to an observed evaluation component rather than a general interactive API.  A non-root node has conditional probability kernel $K_v$:
\[
Z_v\mid(Z_{\pa(v)},\lambda)\sim K_v(\cdot\mid Z_{\pa(v)},\lambda),
\]
which induces the joint DAG law.  Each state space may contain a different kind of object, including a datum, document set, text context, probability law or decision.  \textbf{An edge means that a recorded parent output is supplied to a child;} it is not an attention weight or a claim about hidden reasoning.  The topology is constructed from the versioned trace contract, not inferred from outcome correlations.  Controlled edge removal or substitution and held-out residual checks challenge that contract; an unexplained dependency or a discrepancy beyond its envelope invalidates qualification for that operating domain.

\paragraph{The measured graph.}
Figure~\ref{fig:overview} shows the generic construction: evidence is retrieved, used to generate claims, checked against its sources and passed to a governed action.  When token probabilities are observable, the language component is represented by its autoregressive token chain; tokens refine that component rather than becoming independent workflow stages.  Prespecified wording changes test presentation stability; source removal or substitution tests evidence dependence.  A semantic-state map and external calibration are optional.  Detailed continuation, lexical-equivalence and provenance definitions appear in Appendices~\ref{app:lexical} and~\ref{app:provenance}.

\paragraph{SSEG is observation-adaptive, not open-weight-only.}
An \textbf{API-compatible SSEG} operates on observable workflow measurements such as retrieval scores, claim--source checks, tool traces, controlled interventions and externally supplied node risks.  A \textbf{logprob-enhanced SSEG} additionally refines a language node with returned token or complete-phrase probabilities, keeping truncated residual mass explicit.  A \textbf{self-hosted open-weight SSEG} can further retain full logits, hidden states and instrumented interventions.  Only the latter two profiles expose a token subgraph; open weights are therefore necessary for the deepest lexical and mechanistic audits, not for graph propagation itself.

\paragraph{Source-specific uncertainty.}
For each generated claim $U_i$ and retained source $S_j$, the evaluator scores one of four relation labels: \emph{support}, \emph{qualify}, \emph{contradict} or \emph{insufficient}.  Let $\mathcal H$ denote this four-label set, and let $H_{ij}\in\mathcal H$ denote the claim--source relation.  Define its score by
\[
q_{ij}^{(h)} := q_{ij}(h) := \mathbb P_H\!\left(H_{ij}=h \mid U_i,S_j,\lambda_{\rm eval}\right),
\qquad h\in\mathcal H,
\qquad \sum_{h\in\mathcal H}q_{ij}^{(h)}=1.
\]
Here $\mathbb P_H$ is the fixed evaluator-induced probability law conditional on the claim, source and evaluator settings.  These are model-produced scores, not human ground-truth labels.  The evaluator may be a pinned entailment model or a complete continuation law from an open-weight model; if it returns only one hard label, SSEG records a point mass and does not claim additional relation uncertainty.  A separate structural-reliance score is $r_{ij}=|\mathcal T_i|^{-1}\sum_{k\in\mathcal T_i}[\log p(t_k\mid P,S_j,t_{<k})-\log p(t_k\mid P,\widetilde S_j,t_{<k})]$, where $\mathcal T_i$ indexes the aligned tokens of claim $U_i$ and $\widetilde S_j$ is its fixed, label-blind control passage.  Relation and reliance are retained as distinct typed measurements together with the path that produced them; high semantic support does not by itself show that the generated claim depended on that source.  Scores from different sources are not multiplied because repeated or dependent passages need not provide independent evidence.  Controlled changes to retrieval, source content and presentation remain separate until a declared governance loss aggregates them.  Thus a citation alone is not verified provenance; the auditable object is the claim--source relation and its path.

\section{A conditional pathwise qualification bound}

\paragraph{Why sensitivity tests are not enough.}
It is commonplace to test a model or workflow by performing sensitivity tests.  Such tests reveal whether an output changes when retrieval, prompt presentation, model version or tool behaviour is perturbed.  However, sensitivity alone does not bound how a local change can propagate to the final action or retain the path responsible for that change.

Nodes represent observable workflow states, while channels specify where controlled perturbations enter the workflow and allow their downstream effects to be measured; for example, a retrieval channel varies the query, corpus or retriever and measures the resulting change at retrieval node $R$ and its descendants.  Holding all other elements of $\lambda$ fixed, a channel contrast perturbs one channel to locate where sensitivity enters the workflow.  Before deployment, the validator freezes a versioned reference workflow and measures its nodewise behaviour over the declared qualification domain using controlled benchmark runs; in synthetic tests this behaviour may instead be known by construction.  Let $K_v^0$ denote the resulting reference conditional law at node $v$, and let $K_v^{\mathrm{op}}$ denote the corresponding law observed in the operational workflow.  The reference is an approved comparison baseline, not necessarily factual truth.  For each node $v$, we fix a measurable lower-semicontinuous metric $d_v$ on $\mathcal Z_v$ and let $\mathcal W_v$ be the induced 1-Wasserstein distance.  We use bounded metrics in the experiments; otherwise finite first moments are required.  For parent output $z_{\pa(v)}=(z_p:p\in\pa(v))$, define
\begin{align}
b_v&=\sup_{z_{\pa(v)}}\mathcal W_v\{K_v^{\mathrm{op}}(\cdot\mid z_{\pa(v)}),K_v^0(\cdot\mid z_{\pa(v)})\}, \label{eq:defect}\\
\mathcal W_v\{K_v^0(\cdot\mid z),K_v^0(\cdot\mid z')\}
&\le\sum_{p\in\pa(v)}c_{vp}d_p(z_p,z'_p). \label{eq:influence}
\end{align}
Here $b_v$ is the largest local discrepancy at node $v$ over the declared parent domain.  Equation~\eqref{eq:influence} assumes finite $c_{vp}\ge0$ that bound how strongly a change at parent $p$ can affect node $v$.  Put simply, $b_v$ measures the local mismatch and $c_{vp}$ measures its possible transmission along one edge.

\paragraph{Main guarantee.}
Governance therefore needs to turn the nodewise discrepancies and edgewise transmission coefficients above into a bound on how much uncertainty can reach the terminal action, while retaining which paths contributed to that bound.
The following result is a DAG specialization of standard coupling and Dobrushin-style perturbation arguments \cite{dobrushin1956,makur2020}; the inequality itself is not claimed as new.  The methodological contribution is to connect typed workflow interventions to its local terms, preserve their nodewise decomposition and use the resulting bound for conditional qualification.

\begin{theorem}[Evidence-graph path certificate]
\label{thm:graph}
For nodes $v,T\in V$, write $v\preceq T$ when $v=T$ or a directed path connects $v$ to terminal node $T$.  Let a path be $\gamma=(n_0=v,n_1,\ldots,n_k=T)$, with edge set $E(\gamma)=\{(n_{j-1},n_j):1\le j\le k\}$.  Define its total influence by
\[
\Gamma_{vT}=\sum_{\gamma:v\rightsquigarrow T}\ \prod_{j=1}^{k}c_{n_j n_{j-1}},
\]
where, by convention, the zero-length path gives $\Gamma_{TT}=1$.  Then
\begin{equation}
\mathcal W_T\{\mathcal L(Z_T^{\mathrm{op}}),\mathcal L(Z_T^0)\}
\le B_T:=\sum_{v\preceq T}\Gamma_{vT}b_v. \label{eq:pathsum}
\end{equation}
\end{theorem}

\paragraph{How to read the theorem.}
A proof and a robust extension for uncertain graph structure appear in Appendix~\ref{app:proof}.  \textbf{The proof shows that the terminal discrepancy is no larger than the sum, over the graph, of each local discrepancy multiplied by its downstream path influence.}  The workflow must abstain or seek review if the permitted action changes anywhere within this envelope.

\paragraph{Testing the theory.}
We first generate 35,000 controlled DAG cases: 5,000 seeded simulations in each of seven misspecification scenarios covering edge, weight, node-type, local-measurement, missing-telemetry and combined structural errors.  Each case begins from a known reference DAG and conditional node laws; the designated scenario then injects declared node or edge defects before the operational terminal law is recomputed.  Because the reference graph, injected local discrepancies and realized terminal discrepancy are known by construction, these cases test whether the bound covers the realized discrepancy; the target is deterministic coverage under the stated assumptions, not a statistical prediction interval.  We then test diagnosis on two external benchmarks.  MAGIC (Multi-Hop and Graph-Based Benchmark for Inter-Context Conflicts) creates labelled multi-step conflicts in retrieval-augmented generation \cite{lee2025magic}.  We intervene on its nodes, path-localize and rank the retained routes by their measured contribution and compare the ranking with withheld conflict-path labels.  ALCE (Automatic LLM Citation Evaluation) provides human labels for whether cited evidence supports generated claims \cite{alce2023}.  We replace individual sources, measure the response and compare it with those labels.  MAGIC and ALCE test practical localization; they do not certify the theorem because the complete organic workflow and its true local bounds are not observed.

\paragraph{From bound to governance action.}
The contribution $\rho_{vT}=\Gamma_{vT}b_v$ ranks where review can most reduce the terminal bound: a moderate defect on a strongly amplified path may matter more than a large isolated defect.  Operationally, ELRG first \emph{exposes} a terminal-envelope violation, then \emph{path-localizes} and ranks its largest node and path contributions, \emph{routes} the case to the retained source or component identity and finally \emph{governs} release against the declared tolerance.  Missing telemetry forces review.  After an intervention, recompute the changed node and all recorded descendants before applying the gate again.  This ranking is not a causal diagnosis, and the certificate does not establish factual truth.  A pass remains conditional on the declared graph, operating domain, model vintage, interventions and tolerances.  Appendix~\ref{app:proof} gives the robust corollary, and Algorithm~\ref{alg:sseg} gives the end-to-end procedure for constructing, qualifying and running an SSEG rather than assuming that its local measurements are already available.

\section{Evidence-path governance experiments}

\paragraph{One claim, one evidence ladder.}
The goal of the following experiments is to test one claim: retaining dependency structure turns a terminal warning into an actionable governance artifact.  The controlled study tests the bound; external benchmarks then test path localization, routing and release.  These heterogeneous endpoints are deliberately not pooled.  They form successive links in the same operational chain rather than a claim that SSEG is the best classifier on every task.  Table~\ref{tab:headline-ladder} states the comparator, endpoint and decision consequence at each stage.

\paragraph{What each benchmark contributes.}
Before the external benchmarks, a 35,000-case controlled study tests the pathwise bound where the true graph and injected defects are known.  Table~\ref{tab:main-study-design} then maps each external benchmark to one ELRG stage and endpoint.  The full experimental inventory and secondary studies are reported in Appendix~\ref{app:additional-experiments}.

\begin{table}[t]
\centering
\scriptsize
\setlength{\tabcolsep}{3.5pt}
\caption{Main study design.  Each benchmark tests one stage; endpoints are not pooled.}
\label{tab:main-study-design}
\begin{tabular}{p{.10\linewidth}p{.22\linewidth}p{.25\linewidth}p{.33\linewidth}}
\toprule
Stage & Study & Population & Endpoint and role\\
\midrule
Expose & Controlled stress & 35,000 known-truth cases & Bound coverage and false passes under graph misspecification.\\
Localize & MAGIC \cite{lee2025magic} & 960 cases; 8,592 paths & Rank-1 localization of labelled conflict paths.\\
Route & ALCE; VitaminC \cite{alce2023,vitaminc2021} & 1,467 links; 688 revisions & Claim--source reliance and affected-path recall.\\
Govern & FinGovBench \cite{fingovbench2026} & 1,200 cases; 600 pairs & Safe recall, unsafe release and exact-pair accuracy.\\
\bottomrule
\end{tabular}
\end{table}

\paragraph{Measuring the topology effect.}
In order to isolate the effect of graph topology, we compare SSEG with a measurement-matched flat ablation.  ``Flat'' does not mean uncertainty-free.  A flat rule may use confidence, entropy, entailment or calibrated probabilities, but it does not retain directed adjacency, path identity or descendant scope.  The flat ablation receives the same local uncertainty measurements as SSEG and deliberately discards their topology.  This isolates the value of graph structure.

\paragraph{Measurements and audit.}
The main studies use probabilities computed from model logits or returned log probabilities, changes under controlled source or prompt interventions, and explicitly identified external evaluators.  External scores remain fallible graph nodes rather than ground truth.  No main result asks a model to print its own confidence.  Development data fix every score, threshold and model before the corresponding held-out test labels are opened, and paired intervals resample the natural case, question or topology group.  The principal endpoints are respectively bound coverage and false passes; top-ranked conflict-path accuracy; claim--source reliance and affected-path recall; and safe-release recall, unsafe releases and exact matched-pair accuracy.  Comparisons with measurement-matched flat ablations isolate graph structure; comparisons with external evaluators are labelled operational rather than information-matched.

The controlled study requires no language model.  MAGIC and ALCE use the pinned model families listed in Appendix~\ref{app:model-compute}; FinGovBench uses GPT-OSS-20B and an unchanged-gate Qwen3-8B transfer.  Appendix~\ref{app:model-compute} also reports checkpoint revisions, hardware, timings and profiling limitations.

\paragraph{Expose: can a hidden defect pass a terminal gate?}
The \verb|expose| result in Table~\ref{tab:headline-ladder} uses 35,000 known-truth graph-misspecification cases: the robust path envelope covers every realized terminal discrepancy and produces no false automatic pass.  A nominal terminal envelope false-passes 42.2\% of combined-stress cases.  This validates the conditional calculation under known local defects and topology errors; it does not estimate the prevalence of such faults in deployment.

\paragraph{Path-localize: can the responsible path be found?}
The \verb|localize| result in Table~\ref{tab:headline-ladder} uses 960 held-out MAGIC cases \cite{lee2025magic}, never used for score selection, with 8,592 root-to-answer candidate paths and withheld single- or multi-hop conflict labels.  A frozen SSEG score combines the same lexical node used by the comparator with controlled probability changes along the retained path.  Rank-1 accuracy is the fraction of cases in which the highest-scored candidate is a labelled conflicting path.  It rises from 75.8\% to 80.3--82.3\% across the five models; every case-paired bootstrap 95\% interval for the gain excludes zero.  The result supports localization of the labelled conflict path, not a uniquely causal triple.

\paragraph{Route: can an alarm name the evidence requiring review?}
The \verb|route| result in Table~\ref{tab:headline-ladder} uses 1,467 held-out ALCE claim--source links \cite{alce2023}.  We measure source reliance as the change in aligned answer-token log probability when the cited passage is replaced by a fixed, label-blind control.  Unsupported links show 4.8--8.2 times weaker median reliance than supported links across five model families.  Yet 31.2--38.0\% of human-supported links also trigger structural review because the answer depends weakly or unstably on the source.  The graph therefore separates \emph{support} from \emph{evidential tether}: it does not relabel those links, but names the claim--source edge to verify.  In a separate replay of 688 genuine VitaminC revisions, affected-path recall is 688/688; SSEG reviews 688 rather than 2,528 claims---72.8\% fewer than exhaustive page review---with the same NLI release decisions.

\paragraph{Govern: does topology improve the complete release decision?}
The \verb|govern| result uses a new sealed FinGovBench extension \cite{fingovbench2026} with 600 matched topology pairs.  Each pair holds the event and local workflow components fixed while changing whether the affected node reaches a governed action; 540 cases are release-eligible and 660 require withholding, including 120 non-compensable red-line cases.  Self-hosted GPT-OSS-20B supplies exhaustive continuation laws from causal logits for a primary prompt, an information-equivalent presentation and a matched neutral intervention.  The release map and acceptance criteria are frozen before these 3,600 laws are measured and before any label is opened.  Further details of the FinGovBench design, evaluation and residual cases are provided in Appendix~\ref{app:fingov-safety}.

GPT-OSS-20B is frozen and evaluated zero-shot: it is not fine-tuned on FinGovBench, and SSEG does not alter its weights or generated continuation laws.  The graph-aware release map is fitted only on a disjoint development population and then frozen before the sealed cases are run.  The model-only and SSEG conditions therefore share the same model outputs; what differs is whether the release decision uses the retained workflow topology and propagated path evidence.  This makes the comparison a test of governance value, not a claim that SSEG improves the model's underlying predictive ability.

This design does not ask GPT-OSS to print a verdict on whether its own answer is trustworthy.  GPT-OSS acts as a local measurement instrument: its causal logits supply candidate-decision and defect probabilities, while a deterministic, externally frozen SSEG rule combines those measurements with recorded topology to make the release decision.  Nevertheless, the validation is not model-independent, because the governed model supplies both the candidate and its principal local uncertainty measurements.  A confidently wrong local law can therefore pass through the graph, as the eight residual errors show.

Table~\ref{tab:fingov-confirmation} reports both model evaluations.  For GPT-OSS, the model alone releases 452/660 unsafe cases.  Flat-UQ avoids almost all release---six safe and one unsafe case---and therefore has 1.1\% safe-release recall.  SSEG releases all 540 safe cases, withholds 652/660 unsafe cases and classifies 592/600 topology pairs exactly.  Its unsafe-release rate is 1.21\%, a 98.2\% relative reduction from the model-alone rate of 68.5\%; all 120 hard red-line cases are withheld.  Safe-release recall improves over flat-UQ by 98.9 points (paired topology-bootstrap 95\% interval $[98.0,99.6]$).  The prespecified production-strength requirement of zero observed unsafe releases is not met, so this is a strong confirmatory governance result rather than a universal safety certificate.  The next analysis tests whether the same frozen gate transfers to a second model without refitting.

\paragraph{Cross-model transfer through the unchanged gate.}
The lower panel of Table~\ref{tab:fingov-confirmation} reports the transfer result.  We froze Qwen3-8B predictions before evaluating the same 1,200 FinGovBench labels.  Qwen supplies the same three candidate-event causal-logit measurements as GPT-OSS, with all probability outside the declared family retained as residual mass; neither its weights nor the previously fitted release maps are changed.  The model-only rule releases no case and is therefore safe but useless.  Flat-UQ releases 523/540 safe cases but also 520/660 unsafe cases and classifies only 80/600 matched topology pairs exactly.  SSEG releases all 540 safe cases, releases none of the 660 unsafe cases and classifies all 600 pairs exactly.  Its 3.15-point safe-recall gain over flat-UQ has a topology-clustered 95\% interval of $[1.82,4.67]$ points.  This is an out-of-model transfer on the existing confirmation population, not a second newly sealed population, but it directly shows that the graph-aware gate is not tied to GPT-OSS outputs.

The observation interface was itself prespecified and audited.  Qwen passed a 25-case, label-free qualification requiring at least 80\% probability mass on the declared candidate family under every primary, equivalent and neutral presentation.  In the full run, 2 primary and 162 equivalent presentations fell below that level, while no neutral presentation did; these cases remain in the evaluation and their residual mass remains explicit.  Mistral-7B-Instruct-v0.3 and OLMo-2-7B-Instruct did not pass the pilot interface check: their immediate code-plus-terminator events omitted substantial model probability mass for at least one presentation, so no end-to-end result is reported for them.  Access to logits alone is therefore insufficient; the declared continuation family must first be shown to represent the model's response law.  These failures limit the present cross-model claim rather than being repaired by renormalizing away missing mass.

\begin{table}[t]
\centering
\scriptsize
\setlength{\tabcolsep}{3.2pt}
\caption{Sealed FinGovBench results on 1,200 cases (540 release-eligible; 660 requiring withholding).  Exact-pair accuracy requires both members of a matched topology pair to be correct.}
\begin{tabular}{llrrr}
\toprule
Model & Rule & Safe recall & Unsafe release & Exact pairs\\
\midrule
GPT-OSS-20B & Model only & 82.6\% & 452/660 & 20.5\%\\
 & Flat-UQ & 1.1\% & 1/660 & 11.0\%\\
 & SSEG & \textbf{100.0\%} & \textbf{8/660} & \textbf{98.7\%}\\
\midrule
Qwen3-8B & Model only & 0.0\% & 0/660 & 10.0\%\\
 & Flat-UQ & 96.9\% & 520/660 & 13.3\%\\
 & SSEG & \textbf{100.0\%} & \textbf{0/660} & \textbf{100.0\%}\\
\bottomrule
\end{tabular}
\label{tab:fingov-confirmation}
\end{table}

\paragraph{The eight residual icebergs are diagnostically coherent.}
SSEG unsafely releases eight cases: four involve conflicting evidence, two monitoring drift, one evaluator coupling and one calculation error.  None involves stale sources, unit-scale error, permission or population expansion, prompt-policy drift or a hard red line.  In every residual case, GPT-OSS selects \texttt{NONE} under both equivalent presentations.  Relative to the 652 correctly withheld unsafe cases, the equivalent-prompt laws have lower entropy, larger winning margins, larger \texttt{NONE} mass and lower propagated auditable risk.  The error is therefore interpretable: the graph is present, but its local model measurement is confidently wrong about whether a defect exists.  This is exactly the limitation predicted by the conditional theory---propagation cannot recover a local defect that the observation channel fails to measure.

\paragraph{Secondary evidence.}
The appendix studies test whether the same path representation remains useful outside the four headline stages.  On 96 ToolSandbox executions across two agent models, retained paths reduce post-alarm review by 28.8\%, although trace-only SSEG does not improve detection over the flat contract alarm \cite{toolsandbox2025}.  On 560 untouched Self-RAG source interventions, retaining path structure raises average precision for detecting response-destabilizing evidence changes from .283 to .538; exact-hop localization requires paired local edge contrasts \cite{selfrag2024,musique2022}.  RAGTruth provides a weaker but informative limitation: evidence-sensitive spans modestly improve word-level hallucination localization, while summary-level and incomplete-law transfer gates fail \cite{ragtruth2024}.  These results support targeted review after a qualified alarm, not a universal claim of better failure detection; full protocols and negative results appear in Appendix~\ref{app:secondary}.

Taken together, Table~\ref{tab:headline-ladder} shows the progression from exposing a hidden terminal risk, through path localization and source-specific routing, to a complete release decision.  The table should therefore be read as an evidence ladder rather than as a pooled leaderboard across incomparable tasks.

\begin{table*}[t]
\centering
\small
\setlength{\tabcolsep}{3.5pt}
\caption{The main empirical argument.  Each stage in the ELRG framework tests a different consequence of retaining dependency structure.  Note that heterogeneous endpoints are not pooled.}
\begin{tabular}{>{\centering\arraybackslash}m{.065\textwidth}>{\raggedright\arraybackslash}m{.145\textwidth}>{\raggedright\arraybackslash}m{.18\textwidth}>{\raggedright\arraybackslash}m{.22\textwidth}>{\raggedright\arraybackslash}m{.25\textwidth}}
\toprule
Stage & Iceberg question & Governance without SSEG & Governance with SSEG & Decision consequence\\
\midrule
\shortstack{\stageletter{E}\\[-1pt]\stageletter{X}\\[-1pt]\stageletter{P}\\[-1pt]\stageletter{O}\\[-1pt]\stageletter{S}\\[-1pt]\stageletter{E}} & Can hidden defects pass? & Nominal terminal gate false-passes 42.2\% under combined stress. & 100\% observed bound coverage and zero false passes across 35,000 controlled cases. & A terminal pass is insufficient: widen the envelope or escalate when structural uncertainty is unresolved.\\
\midrule
\shortstack{\stageletter{L}\\[-1pt]\stageletter{O}\\[-1pt]\stageletter{C}\\[-1pt]\stageletter{A}\\[-1pt]\stageletter{L}\\[-1pt]\stageletter{I}\\[-1pt]\stageletter{Z}\\[-1pt]\stageletter{E}} & Can the conflict path be found? & Lexical rank-1 is 75.8\%. & MAGIC rank-1 is 80.3--82.3\% across five model families; all paired gains are positive. & Inspect the highest-contribution path first instead of searching every evidence path.\\
\midrule
\shortstack{\stageletter{R}\\[-1pt]\stageletter{O}\\[-1pt]\stageletter{U}\\[-1pt]\stageletter{T}\\[-1pt]\stageletter{E}} & Which source should be reviewed? & HHEM-plus-lexical flat rank-1 is 58.7\%; RAGAS is 61.3\%. & ALCE rank-1 is 68.0\%; on 688 VitaminC revisions, SSEG finds every affected claim path while reviewing 72.8\% fewer claims than exhaustive NLI. & Route review to the named claim--source edge and revalidate its dependent claims.\\
\midrule
\shortstack{\stageletter{G}\\[-1pt]\stageletter{O}\\[-1pt]\stageletter{V}\\[-1pt]\stageletter{E}\\[-1pt]\stageletter{R}\\[-1pt]\stageletter{N}} & Can a closed-loop decision be released safely? & GPT-OSS alone releases 452/660 unsafe cases; flat-UQ releases almost nothing. & SSEG releases all 540 safe cases and withholds 652/660 unsafe cases with GPT-OSS; unchanged-gate Qwen transfer releases 540/540 safe and 0/660 unsafe. & Use the graph-aware gate for selective release; inspect residual cases as failures of local defect measurement.\\
\bottomrule
\end{tabular}
\label{tab:headline-ladder}
\end{table*}

\section{Conclusion}
The evidence supports a single conclusion: terminal validity and path integrity are distinct, but they can be governed together.  The known-truth study validates conditional propagation; MAGIC validates path localization; ALCE and VitaminC validate source-specific routing; and the sealed FinGovBench extension shows that the same representation can materially improve a complete release decision.  The latter result is especially informative because it includes both success and failure.  SSEG detects 98.8\% of unsafe GPT-OSS cases without withholding a safe case, yet eight confident local no-fault measurements still pass.  Under an unchanged-gate Qwen transfer, it releases every safe case and no unsafe case, whereas flat-UQ releases 520 unsafe cases.  The graph therefore adds more than an explanation attached after a decision: it changes which actions are released and transfers beyond the model used to fit the gate.  At the same time, the GPT-OSS residual failures and the Mistral/OLMo interface failures reveal two precise remaining dependencies---the quality of the local observation law and whether the declared continuation family captures it.  The operational chain is consequently explicit: \textbf{expose the iceberg, localize its path, route the affected evidence or component and govern release.}

\clearpage

\section*{Reproducibility statement}
The appendices and reproducibility archive contain frozen manifests, model revisions, prompts, complete-phrase records, analysis code and machine-readable summaries.  For FinGovBench they also retain the preregistration, separately sealed labels, 3,600 causal-logit laws, frozen release map, SHA-256 prediction freeze, final evaluation and post-seal residual analysis.  A claim ledger maps every headline result to its artifact and field.  Downloaded corpora and caches are reconstructed by documented scripts; failed gates and invalidated pilots are retained.

\section*{AI use and ethics statement}
Generative AI tools assisted with manuscript editing, experiment-code drafting, literature discovery and formatting.  The authors wrote and revised the paper themselves; AI-generated prose was not accepted as evidence or authorship.  Every citation was checked by the authors against its source, and the data, analysis code, intermediate artifacts and reported results went through repeated rounds of human review.  The authors executed the analyses, inspected the underlying records and take responsibility for every claim.  The study uses non-personal benchmark records and aims to improve uncertainty measurement in consequential AI workflows.  No universal safety or causal claim is made.

\bibliographystyle{plain}

\fi

\ifdefined\mainonly
\else
\ifdefined\supplementonly
\title{Finding Icebergs in Language-Model Workflows\\[0.3em]\large Supplementary Material}
\maketitle
\paragraph{Scope.}
This supplement contains proofs, expanded protocols, limitation studies, negative results and observation-profile details supporting the main manuscript.  These materials are separated from the nine-page paper to keep its empirical claim focused.  The common-control four-panel iceberg visualization will be restored after the corrected Mistral run replaces the single panel affected by a duplicated beginning-of-sequence token; the other audited model outputs do not require rerunning.
\fi
\ifdefined\supplementonly
\else
\clearpage
\fi
\appendix
\raggedbottom
\section{Anonymous reproducibility archive}
\label{app:reproducibility-archive}

The purpose of this section is to provide a self-contained route for verifying the paper's headline results and locating the corresponding implementation artifacts.

Code, cached case-level records, frozen result artifacts, integrity manifests and tests are available in the anonymous review repository at
\begin{center}
\url{https://anonymous.4open.science/r/sseg-review-artifacts-1FD6/}.
\end{center}
The archive contains visible local implementations of graph construction, candidate-law normalization, uncertainty propagation, governance scoring, evaluation and bootstrap analysis; the paper-reproduction path does not depend on a hidden API service.  It distinguishes recomputation from artifact verification and from a new live replication, so that results generated under a different model, tokenizer, prompt family or provider are not silently mixed with the frozen paper evidence.

\paragraph{Implementation map and minimum environment.}
The reference implementation targets Python 3.10 or later and was checked with Python 3.13.9.  Its core offline dependencies are NumPy 1.24 or later and SciPy 1.10 or later; Pytest 8 or later runs the test suite.  The SSEG data structure and path-propagation calculation are implemented in \texttt{src/sseg\_repro/graph.py}; \texttt{src/sseg\_repro/fingov.py} reconstructs the sealed governance metrics and clustered bootstrap; \texttt{src/sseg\_repro/artifacts.py} enforces the SHA-256 and schema contract; and \texttt{src/sseg\_repro/cli.py} provides the command-line entry point.  In particular, \texttt{graph.py} materializes declared parent--child dependencies, computes local-to-terminal path contributions and returns the propagated terminal bound and ranked review paths.  It does not infer graph edges from attention or correlations.

\paragraph{Minimal headline-results workflow.}
From the repository root, the reviewer can run:
\begin{verbatim}
python -m venv .venv
source .venv/bin/activate
python -m pip install -e '.[test]'
sseg-repro paper --bootstrap 20000
pytest -q
\end{verbatim}
The \texttt{paper} command first verifies 25 reconstruction inputs (91,566,909 bytes) against \texttt{reconstruction\_manifest.json}.  It then regenerates the 35,000 controlled cases; reconstructs MAGIC path rankings from 54,430 frozen conditional-law records; recomputes the five-model ALCE statistics and the 75-question operational routing comparison from 1,467 link-level records; replays all 688 VitaminC updates over 440 sealed workflows; recomputes the RAGTruth word-level and ToolSandbox case-level results; and reconstructs both FinGovBench evaluations from case-level predictions and labels.  It repeats the declared 20,000-resample analyses, fails closed on any manuscript-contract mismatch and writes machine-readable comparisons, manuscript-style tables, figures and a visual report under \texttt{paper\_outputs/}.  This begins after model acquisition: repeating an acquisition stage still requires the corresponding public benchmark, pinned checkpoint or evaluator environment.

The repository supports three modes:
\begin{enumerate}[leftmargin=*]
    \item \textbf{Offline paper reproduction (default).}  This mode requires no API key or model-provider call.  It verifies every reconstruction input, recomputes all current headline metrics from case-level records or frozen conditional laws, regenerates the displayed tables and figures, enforces the frozen manuscript-value contract and runs the test suite.  It does not claim to repeat the earlier language-model acquisition stage.
    \item \textbf{Local open-weight measurement.}  A user may supply a pinned Hugging Face causal-language-model checkpoint, including the paper's GPT-OSS-20B revision.  The local adapter teacher-forces every declared candidate continuation under the same prefix, records token-level log probabilities and normalizes the exhaustive candidate family before applying the visible SSEG code.  The resulting observations constitute a new replication unless they reproduce the declared checkpoint and protocol.
    \item \textbf{API-compatible observation profile.}  An optional OpenAI-compatible runner records the request configuration, response identifier, generated text and any returned generated-token log probabilities.  Because such an endpoint need not expose weights, an immutable checkpoint or a complete competing-phrase law, this mode is provided for observation-profile comparisons and is not treated as a reproduction of the paper's open-weight token-law result.
\end{enumerate}
\section{SSEG algorithm and core protocol}
\label{app:protocols}

The purpose of this section is to state the implementable SSEG algorithm, its inputs and outputs, and the qualification and runtime protocol used in the experiments.

\subsection{An implementable SSEG algorithm}
\label{app:algorithm}

Algorithm~\ref{alg:sseg} gives the qualification and runtime protocol for an implementable SSEG.

An implementation requires only three declared objects: an observable workflow $\mathcal W_{\rm obs}$ from which a typed acyclic graph $\overline G$ with terminal node $T$ can be traced; node measurement rules $M_v$ and controlled interventions for estimating local discrepancies and edge transmission; and a governance policy $\Pi$ containing the operating domain, hard controls, terminal tolerance and permitted actions.  Here $\mathcal W_{\rm obs}$ means the deployed chain of retrievers, prompts, model calls, tools, evaluators, data transfers and terminal decision rules---not the language model's weight matrices.  A pinned model identifier or weight revision is recorded as component metadata when available.  An edge denotes an observed transfer of a parent output to a child, not correlation, attention or narrative relevance.  Cycles are unrolled in time.

Qualification precedes the governed case, and the frozen result is reused across runtime traces.  A runtime trace populates the graph, propagates uncertainty and produces the ELRG decision.  ELRG is therefore the operational reading of the populated SSEG, not a separate scoring model.  After node measurement, propagation is $O(|V|+|E|)$; model calls dominate cost.  A missing measurement or unqualified dependency triggers review rather than contributing zero uncertainty.

\begin{algorithm}[t]
\small
\DontPrintSemicolon
\caption{Construct and run a stochastic semantic evidence graph}
\label{alg:sseg}
\KwIn{Observable workflow $\mathcal W_{\rm obs}$ (components, data transfers and versions; not model weights); disjoint qualification records $(D_{\rm fit},D_{\rm val})$; policy $\Pi$; versioned runtime trace $x$}
\KwOut{Governance action $a$ and reproducible SSEG certificate $C_x$}
\tcp{Qualification: construct and freeze the SSEG}
$\overline G\leftarrow\textsc{TraceGraph}(\mathcal W_{\rm obs},D_{\rm fit})$\;
\ForEach{$v\in V(\overline G)$}{
  $(M_v,d_v,\mathcal I_v)\leftarrow\textsc{QualifyNode}(v,D_{\rm fit})$\;
  $b_v^+\leftarrow\textsc{SimultaneousBound}(M_v,d_v,\mathcal I_v,D_{\rm val})$\;
  \If{$b_v^+$ is unidentified}{\textsc{MarkUnqualified}$(v)$\;}
}
\ForEach{edge $p\rightarrow v$ in $\overline G$}{
  $c_{vp}^+\leftarrow\textsc{InfluenceBound}(p,v,\mathcal I_p,D_{\rm val})$\;
  \If{$c_{vp}^+$ is unidentified}{\textsc{MarkUnqualified}$(p\rightarrow v)$\;}
}
$\Pi\leftarrow\textsc{FitPolicy}(D_{\rm fit})$\;
$q\leftarrow\textsc{Validate}(\overline G,\Pi,D_{\rm val})$
\tcp*[r]{coverage, unsafe release, availability}
$\mathcal S^\star\leftarrow\textsc{Freeze}(\overline G,\{M_v,b_v^+,c_{vp}^+\},\Pi,q,\text{versions},\text{hashes})$\;
\BlankLine
\tcp{Runtime: populate the frozen SSEG and govern the action}
$G_x\leftarrow\textsc{Materialize}(x,\mathcal S^\star)$\;
\If{$G_x$ has missing telemetry, unqualified dependencies, version mismatches or lies outside the operating domain}{
  \Return{$(\textsc{Review},\textsc{UnresolvedCertificate}(G_x))$}\;
}
\ForEach{$v\in V(G_x)$}{
  $b_v^+(x)\leftarrow M_v(x)$\;
  \If{$b_v^+(x)$ is unavailable}{\Return{$(\textsc{Review},\textsc{UnresolvedCertificate}(v))$}\;}
}
$g_T\leftarrow1$\;
\ForEach{$v\ne T$ in reverse topological order}{
  $g_v\leftarrow\sum_{(v,w)\in E(G_x)}c_{wv}^+g_w$
  \tcp*[r]{downstream influence}
}
\ForEach{$v\preceq T$}{$\rho_{vT}\leftarrow g_vb_v^+(x)$\;}
$B_T^+\leftarrow\sum_{v\preceq T}\rho_{vT}$
\tcp*[r]{terminal uncertainty bound}
$h\leftarrow(B_T^+>\tau)\lor\textsc{HardControlFailure}(G_x)$
\tcp*[r]{Expose}
$L\leftarrow\textsc{RankPaths}(G_x,\{\rho_{vT}\})$
\tcp*[r]{Localize}
$\mathcal R\leftarrow\textsc{Route}(L)$, where $\mathcal R(v)=\{w:v\rightsquigarrow w\rightsquigarrow T\}$\;
$e\leftarrow\textsc{HardControlsPass}(G_x)\land[(B_T^+\le\tau)\lor\textsc{InvariantAction}(G_x,\mathcal S^\star)]$\;
\eIf{$e$}{
  $a\leftarrow\textsc{Release}$\;
}{
  $a\leftarrow\Pi.\textsc{FailAction}$\;
}
$C_x\leftarrow\{a,h,B_T^+,\rho_{vT},L,\mathcal R,\text{versions},\text{hashes}\}$\;
\If{node $v$ is repaired}{\textsc{Invalidate}$(\mathcal R(v))$; \textsc{Rerun}$(\mathcal R(v))$\;}
\Return{$(a,C_x)$}\;
\end{algorithm}

Table~\ref{tab:main-study-design} gives the four-stage main design.  Appendix~\ref{app:secondary} summarizes the secondary transfer and limitation studies; complete populations, splits, prompts, thresholds and provider outputs remain frozen in the anonymous archive.

\paragraph{Financial robustness audit.}
The final partition spans April 2025--July 2026.  With thresholds unchanged, SSEG accepts 65.0\%, 45.0\%, 42.1\% and 47.4\% across four chronological blocks and has no false acceptance in any block.  Over all 39 material-discrepancy cases, zero observed false acceptances imply a one-sided 95\% Clopper--Pearson upper bound of 7.4\%.  A paired circular moving-block bootstrap over chronologically ordered dates gives SSEG acceptance-gain intervals of [.269,.526] over terminal entropy and [.295,.538] over generic input drift at block length eight; lengths four and twelve give the same positive conclusion.  Multiplying the frozen SSEG threshold by .75 leaves 47.4\% acceptance and no false acceptance; multiplying it by 1.25 gives 51.3\% acceptance and one false acceptance.  These checks describe the original discrepancy gate, but a direct terminal-drift score identifies a label defined by that same drift perfectly and is therefore not an independent comparator.  The main text corrects this limitation with a separate reference-state audit and makes no superiority claim over direct terminal monitoring.  All block-bootstrap and threshold-sensitivity checks were computed after the final partition was opened and are descriptive, not additional confirmatory endpoints.

\section{Expanded measurement and governance construction}
\label{app:construction}

The purpose of this section is to define the detailed SSEG measurement, dependence and governance construction.  Appendix~\ref{app:protocols} gives the operational algorithm---what is fitted, frozen, populated and returned.  This section builds on it by specifying the probability laws, phrase measurements, joint-dependence set and robust governance objective that give those algorithmic objects their mathematical meaning.

For the detailed SSEG construction, write $\lambda_{\rm ret}$ for retrieval and ranking settings, $\lambda_{\rm pres}$ for prompt presentation, $\lambda_{\rm model}$ for model and decoding settings, $\lambda_{\rm eval}$ for evaluator settings and $\lambda_{\rm gov}$ for the governance rule.  Let $C$ be the complete context, $\bm S$ the retained passages and $t_{1:m}$ a complete continuation of $m$ tokens.  For pinned model parameters $\theta$, the continuation law factorizes as
\begin{equation}
Q_\theta(t_{1:m}\mid C,\bm S)=\prod_{j=1}^{m}K_j(t_j\mid C,\bm S,t_{<j},\lambda_{\rm model}). \label{eq:tokenfactor}
\end{equation}
For prespecified measurable phrase events $\mathcal P=(E_1,\ldots,E_d)$, retain
\begin{equation}
M_{\mathcal P}\{Q(\cdot\mid C,\bm S)\}=\bigl(Q(E_1\mid C,\bm S),\ldots,Q(E_d\mid C,\bm S)\bigr), \label{eq:phrase}
\end{equation}
including uncaptured expressed mass rather than silently renormalizing it.  If $\mathcal G_0$ is a prespecified family of information-preserving context transformations, $B(C)=M_{\mathcal P}\{Q(\cdot\mid C,\bm S)\}$ and $d_B$ is the declared metric, the stability requirement is
\begin{equation}
\sup_{g\in\mathcal G_0}\sup_{C\in D_g}d_B\{B(gC),B(C)\}\le\varepsilon. \label{eq:lexicalstability}
\end{equation}

For claim units $\bm U$, retained passages $\bm S$ and claim--source relations $\bm H$, the source-preserving graph retains
\begin{equation}
K_R(d\bm s\mid\bm y,\lambda_{\rm ret})K_U(d\bm u\mid C,\bm s,\lambda_{\rm pres},\lambda_{\rm model})\Pi_H(d\bm h\mid\bm u,\bm s,\lambda_{\rm eval})K_T(dz\mid\bm u,\bm h,\lambda_{\rm gov}). \label{eq:provenancefactor}
\end{equation}
Separate source marginals do not identify their dependence, so the compatible joint-law set $\mathfrak C$ is retained until the governance node applies declared loss:
\begin{equation}
d^*\in\arg\min_{d\in\mathcal D}\ \sup_{\Pi_H\in\mathfrak C}\E_{\Pi_H}[\ell(d,Z_T)]. \label{eq:governanceloss}
\end{equation}
Controlled channel changes compare one recorded setting at a time while holding the remaining configuration fixed.  Such a contrast is causal only when the intervention is randomized or otherwise identified.

\begin{corollary}[Robust certificate under declared graph uncertainty]
\label{cor:robust}
Let an acyclic supergraph $\overline G$ contain the true graph.  If $b_v\le b_v^+$ and $c_{vp}\le c_{vp}^+$ on $\overline G$, then $\mathcal W_T\{\mathbb P(Z_T^{\mathrm{op}}\in\cdot),\mathbb P(Z_T^0\in\cdot)\}\le B_T^+:=\sum_{v\preceq T}\Gamma_{vT}^+b_v^+$, where $\Gamma_{vT}^+$ uses the upper edge influences.
\end{corollary}

If remediation on $S\subseteq V$ reduces $b_v$ to $\widetilde b_v$ without changing the qualified graph or influence bounds, then
\begin{equation}
\widetilde B_T=B_T-\sum_{v\in S}\Gamma_{vT}(b_v-\widetilde b_v)\le B_T. \label{eq:remediation}
\end{equation}
This reduces the certified upper bound; it is not a causal claim about realized error.  Operationally, one declares the trace graph and loss, obtains simultaneous local allowances on qualification data, freezes gates and actions, ranks $\Gamma_{vT}^+b_v^+$ at evaluation and abstains whenever telemetry is missing or the action changes across the envelope.

\section{Proof of the path certificate}
\label{app:proof}

The purpose of this section is to prove the pathwise uncertainty certificate and its robust extension under declared graph uncertainty.
\begin{proof}[Proof of Theorem~\ref{thm:graph}]
Fix a topological ordering $v_1,\ldots,v_{|V|}$ of the DAG and let $\delta>0$.  The proof proceeds recursively in that order.

\emph{Step 1: couple the root nodes.}
If $v$ is a root, its operational and reference laws are $K_v^{\mathrm{op}}$ and $K_v^0$.  Choose a coupling $(Z_v^{\mathrm{op}},Z_v^0)$ whose expected cost is within $\delta/|V|$ of their Wasserstein distance.  By the definition of $b_v$,
\[
q_v^{(\delta)}:=\E[d_v(Z_v^{\mathrm{op}},Z_v^0)]
\le b_v+\delta/|V|.
\]

\emph{Step 2: compare the two conditional laws at a non-root node.}
Suppose the parent pairs $(Z_p^{\mathrm{op}},Z_p^0)$ have already been coupled for every $p\in\pa(v)$.  Conditional on realized parent values $(z^{\mathrm{op}},z^0)$, insert the intermediate law $K_v^0(\cdot\mid z^{\mathrm{op}})$.  The triangle inequality for $\mathcal W_v$ gives
\begin{align*}
&\mathcal W_v\!\left\{K_v^{\mathrm{op}}(\cdot\mid z^{\mathrm{op}}),
K_v^0(\cdot\mid z^0)\right\}\\
&\quad\le
\mathcal W_v\!\left\{K_v^{\mathrm{op}}(\cdot\mid z^{\mathrm{op}}),
K_v^0(\cdot\mid z^{\mathrm{op}})\right\}
+\mathcal W_v\!\left\{K_v^0(\cdot\mid z^{\mathrm{op}}),
K_v^0(\cdot\mid z^0)\right\}\\
&\quad\le b_v+\sum_{p\in\pa(v)}c_{vp}
d_p(z_p^{\mathrm{op}},z_p^0),
\end{align*}
where the final inequality uses \eqref{eq:defect} and \eqref{eq:influence}.

\emph{Step 3: extend the joint coupling through node $v$.}
For each realized parent pair, choose a conditionally $\delta/|V|$-optimal coupling of the two child laws.  Standard Borel state spaces and lower-semicontinuous costs permit a measurable $\delta$-optimal selection, so its conditional expected cost can be integrated over the already constructed parent coupling.  This yields
\begin{align*}
q_v^{(\delta)}
&\le b_v+\sum_{p\in\pa(v)}c_{vp}
\E[d_p(Z_p^{\mathrm{op}},Z_p^0)]+\delta/|V|\\
&=b_v+\sum_{p\in\pa(v)}c_{vp}q_p^{(\delta)}+\delta/|V|.
\end{align*}
Repeating this construction in topological order produces, for each fixed $\delta>0$, a joint coupling of all deployed and reference nodes satisfying, componentwise,
\[
\bm q^{(\delta)}\le \bm b+\Lambda\bm q^{(\delta)}+(\delta/|V|)\bm 1,
\]
where $\Lambda_{vp}=c_{vp}$ when $p\to v$ and $\Lambda_{vp}=0$ otherwise.

\emph{Step 4: solve the recursive inequality.}
In topological order, $\Lambda$ is strictly lower triangular.  Hence $\Lambda^{|V|}=0$ and
\[
(\bm I-\Lambda)^{-1}=\bm I+\Lambda+\cdots+\Lambda^{|V|-1}.
\]
Every entry of this inverse is nonnegative.  Therefore multiplication of the preceding inequality by $(\bm I-\Lambda)^{-1}$ preserves the componentwise order:
\[
\bm q^{(\delta)}\le(\bm I-\Lambda)^{-1}\{\bm b+(\delta/|V|)\bm 1\}.
\]

\emph{Step 5: identify the inverse with directed paths.}
For $m\ge1$, the entry $(\Lambda^m)_{Tv}$ is the sum, over all length-$m$ directed paths
$\gamma:v\rightsquigarrow T$, of the product of the edge influences along that path:
\[
(\Lambda^m)_{Tv}
=\sum_{\substack{\gamma:v\rightsquigarrow T\\|\gamma|=m}}
\prod_{(p,w)\in\gamma}c_{wp}.
\]
Summing over $m=0,\ldots,|V|-1$ therefore gives
$[(\bm I-\Lambda)^{-1}]_{Tv}=\Gamma_{vT}$, including the length-zero path when $v=T$.  Taking the terminal component of the preceding vector bound yields
\[
q_T^{(\delta)}\le B_T+(\delta/|V|)[(\bm I-\Lambda)^{-1}\bm 1]_T.
\]

\emph{Step 6: return from the constructed coupling to Wasserstein distance.}
The Wasserstein distance is the infimum of expected terminal costs over all couplings.  It is therefore no larger than $q_T^{(\delta)}$ for every $\delta>0$.  Taking the scalar limit $\delta\downarrow0$---without asserting convergence or existence of an exactly optimal joint coupling---gives
\[
\mathcal W_T\{\mathbb P(Z_T^{\mathrm{op}}\in\cdot),\mathbb P(Z_T^0\in\cdot)\}
\le B_T,
\]
which proves \eqref{eq:pathsum}.
\end{proof}

Equation~\eqref{eq:remediation} follows by applying \eqref{eq:pathsum} before and after replacing $b_v$ by $\widetilde b_v$ on $S$ and subtracting the two bounds.

\begin{proof}[Proof of Corollary~\ref{cor:robust}]
Apply Theorem~\ref{thm:graph} to the true graph to obtain $\mathcal W_T\le\sum_v\Gamma_{vT}b_v$.  Embed it in $\overline G$ by assigning zero influence to each supergraph edge absent from the true graph.  Every path product is nonnegative, so replacing each true influence by its upper bound and adding uncertain paths can only increase every path sum: $\Gamma_{vT}\le\Gamma_{vT}^+$.  Since also $0\le b_v\le b_v^+$,
\[
\sum_v\Gamma_{vT}b_v\le\sum_v\Gamma_{vT}^+b_v^+=B_T^+.
\]
Combining the two inequalities proves the claim.  If the true graph or a true local quantity lies outside the declared envelope, this comparison need not hold.
\end{proof}

\section{FinGovBench open-weight sealed confirmation}
\label{app:fingov-safety}

The purpose of this section is to document the sealed FinGovBench design, evaluation protocol, results and residual unsafe releases.

The confirmation population is generated and sealed after the release map has been trained on the earlier development population.  It contains 600 matched topology pairs.  Within a pair, the event and local node descriptions are held fixed while the affected component is either connected to or blocked from the governed action.  Sixty pairs contain non-compensable red-line failures and therefore require withholding in both arms.  The remaining pairs test whether the same measured local event changes the decision when its downstream path changes.  The public records, 1,200 case identifiers and preregistered criteria are frozen before model execution.  All 3,600 laws are then measured zero-shot with self-hosted GPT-OSS-20B at pinned revision \texttt{6cee5e81}; no benchmark fine-tuning or weight update is performed.  Predictions from all three methods are checksummed, and only then are individual labels opened.  Because the methods consume the same saved model laws, their comparison isolates the incremental effect of uncertainty aggregation and graph-aware governance.  The final gate is deterministic and externally frozen, but the study is not an independent-model audit: GPT-OSS supplies both the candidate decision and the local laws used to measure defects.  Table~\ref{tab:fingov-confirmation} also reports an unchanged-gate Qwen transfer on this population; a second independently sealed population remains future work.

As Table~\ref{tab:fingov-confirmation} shows, the GPT-OSS flat-UQ result illustrates why unsafe-release rate cannot be reported without coverage: it is safe largely because it releases almost nothing.  SSEG instead recovers every safe release while reducing the model-alone unsafe-release rate by 98.2\%.  Its safe-recall gain over flat-UQ is 98.9 points, with topology-pair bootstrap interval $[98.0,99.6]$.  All 120 red-line cases are withheld.  Eight unsafe releases remain, so the study fails the deliberately stringent criterion of zero observed unsafe release; the one-sided 95\% upper bound is 2.18\%.

Post-seal characterization does not alter the frozen predictions.  The residual set consists of four conflicting-evidence cases, two monitoring-drift cases, one evaluator-coupling case and one calculation error.  Every residual selects \texttt{NONE} under both equivalent prompts.  Compared with correctly withheld unsafe cases, these laws have higher \texttt{NONE} mass and winning margins, lower entropy and lower propagated auditable risk.  The graph has represented the consequential path, but the local language-model observation assigns it too little defect mass.  This distinguishes a graph failure from a measurement failure and gives a concrete target for further qualification without redefining the sealed endpoint.

\section{Open-weight checkpoints and compute record}
\label{app:model-compute}

The purpose of this section is to identify the model revisions, hardware, measurement workloads and compute limitations needed to interpret and reproduce the open-weight studies, including FinGovBench.

Table~\ref{tab:open-weight-versions} identifies every open-weight checkpoint used to produce a reported language-model measurement.  A revision prefix denotes the immutable Hugging Face commit recorded in the run manifest; the archive retains the complete 40-character revision.  ``Provider snapshot unavailable'' means that the named open-weight model was accessed through a hosted endpoint that did not expose the weight revision.  Such runs are reported as hosted-model limitation studies rather than bitwise-reproducible self-hosted measurements.

\begin{table}[H]
\centering
\scriptsize
\setlength{\tabcolsep}{3.5pt}
\caption{Open-weight model versions used in the experiments.}
\begin{tabular}{p{0.34\linewidth}p{0.20\linewidth}p{0.38\linewidth}}
\toprule
Checkpoint & Immutable revision & Reported use\\
\midrule
\url{openai/gpt-oss-20b} & \texttt{6cee5e81ee83} & FinGovBench, ALCE and complete-phrase audits\\
\url{Qwen/Qwen3.5-9B} & \texttt{c20223623576} & MAGIC and ALCE common-control scoring\\
\url{Qwen/Qwen3-8B} & \texttt{b968826d9c46} & FinGovBench cross-model transfer, phrase-equivalence and post-training studies\\
\href{https://huggingface.co/mistralai/Mistral-7B-Instruct-v0.3}{Mistral-7B-Instruct-v0.3} & \texttt{c170c708c41d} & MAGIC, ALCE, RAGTruth and HotpotQA scoring\\
\url{allenai/OLMo-2-1124-7B-Instruct} & \texttt{470b1fba1ae0} & MAGIC and ALCE common-control scoring\\
\url{unsloth/Meta-Llama-3.1-8B-Instruct} & \texttt{a2856192dd7c} & MAGIC and ALCE common-control scoring\\
\url{unsloth/gemma-2-9b-it} & \texttt{fc7d4737cda1} & MAGIC and ALCE common-control scoring\\
\url{selfrag/selfrag_llama2_7b} & \texttt{190261383b07} & HotpotQA multi-hop path study\\
\url{openai/gpt-oss-120b} via Together & provider snapshot unavailable & VitaminC production-code limitation studies\\
\url{Qwen/Qwen3.5-9B} via Together & provider snapshot unavailable & RAGTruth hosted replication\\
\bottomrule
\end{tabular}
\label{tab:open-weight-versions}
\end{table}

The sealed FinGovBench collection used four independent AWS \texttt{g5.xlarge} workers, each with one NVIDIA A10G GPU (24\,GiB device memory), one model replica and one 300-case shard.  The loader used Transformers \texttt{device\_map=auto} and \texttt{torch\_dtype=auto}; every case required three exhaustive complete-candidate laws.  The four saved logs report 900 newly computed laws per shard and measurement-loop elapsed times of 791.8, 862.5, 853.1 and 809.6 seconds; checkpoint loading took approximately 94--108 seconds per worker.  These are execution records, not a normalized throughput comparison.

The cross-model Qwen3-8B transfer used one AWS \texttt{g5.xlarge} worker with the same 24\,GiB A10G configuration.  It computed 3,600 candidate-event measurements with explicit residual mass for 1,200 cases in 1,117.4 seconds after model loading.  A label-free 25-case qualification preceded the full run; predictions were hashed before the existing confirmation labels were evaluated.  Mistral and OLMo qualification artifacts are retained as negative interface results rather than converted into release metrics.

The subsequent governance computation is lightweight once the laws are frozen.  A 500-case, 20-repeat reference profile on macOS 15.5 ARM with Python 3.13.9 used no external model calls: median per-case times were 0.0068\,ms for flat-UQ, 0.0313\,ms for the lexical rule and 0.0153\,ms for SSEG, with recorded peak Python allocations below 0.01\,MB.  These figures exclude model inference and measure only the release-rule implementation.  Earlier open-weight runs preserve checkpoint, prompt and output manifests but not a complete accelerator inventory; we therefore do not infer missing GPU models or report cross-experiment hardware speedups.

\section{Additional experiments and limitation studies}
\label{app:additional-experiments}
The purpose of this section is to report transfer studies, negative results and empirical limitations that test the breadth of SSEG without enlarging the headline claim.

The following experiments probe breadth and failure limitations but are not part of the headline claim.  They include prompt-equivalence, financial RAG, ALCE and RAGTruth provenance, code-repair and post-training studies.  Their results are retained to prevent selective reporting; weaker or failed comparisons are explicitly scoped.

\paragraph{Open-weight phrase audit.}
At each of 125 points on a $5^3$ evidence lattice, we score three complete phrase events under four presentations that preserve evidence and task.  The three axes are equity breadth, implied volatility and investment-grade credit-spread movement; each takes the ordered levels materially lower, slightly lower, unchanged, slightly higher and materially higher.  GPT-OSS-20B, Qwen3-8B and Mistral-7B use pinned revisions, native chat templates and terminators.  No correctness labels, semantic reduction or printed confidence enter the test.  Pairwise statistics use only measurements passing the frozen 80\% expressed-mass gate.

\begin{figure}[H]
\centering
\includegraphics[width=.98\linewidth]{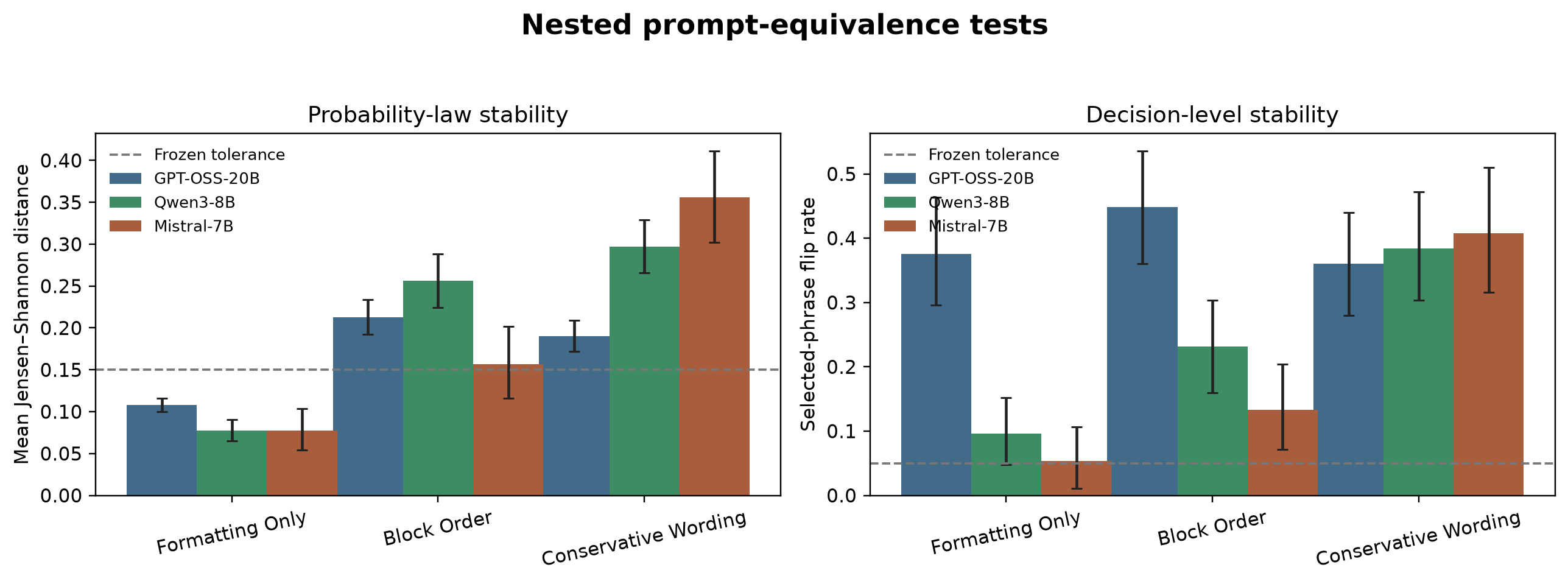}
\caption{Frozen information-equivalence tests on three open-weight architectures.  Bars show mean phrase-law distance and selected-phrase change rate; error bars are 95\% configuration-bootstrap intervals and dashed lines are tolerances.  Every relation fails at least one stability gate.}
\label{fig:nestedequiv}
\end{figure}

\paragraph{Finding 1: equivalent prompts can change the measured output.}
Figure~\ref{fig:nestedequiv} shows that the problem is not isolated to one model or one broad prompt rewrite.  GPT-OSS selected phrases change in 37.6--44.8\% of cases across formatting, block-order and conservative-wording tests.  Qwen changes in 9.6--38.4\%, and Mistral in 5.3--40.8\% among jointly qualified pairs; 42/500 Mistral measurements additionally abstain for insufficient phrase mass.  Canonical phrase laws disagree with GPT-OSS in 50.4\% of Qwen cases and 52.0\% of Mistral cases.  Exact GPT-OSS duplicates agree to machine precision.  \textbf{Takeaway: the audit diagnoses deterministic presentation and model dependence, not sampling noise, and a certificate cannot be transferred across model versions without retesting.}

\paragraph{Propagation and real-workflow diagnosis.}
In the controlled graph, the size of every injected local error and the strength with which it can affect later nodes are set in advance and therefore known exactly.  The real study uses a $2\times2$ design: two point-in-time retrieval rules crossed with two information-equivalent presentations.  An initial 60 dated financial queries are split 20/20/20 for calibration, validation and testing.  A disjoint 120-date set then requalifies deployment thresholds, which are tested once on all 78 remaining eligible dates.  These 198 dates exclude the initial study and produce 792 additional executions.  The contemporary model tests workflow stability, not a historical trading strategy.

\begin{figure*}[t]
\centering
\includegraphics[width=\linewidth]{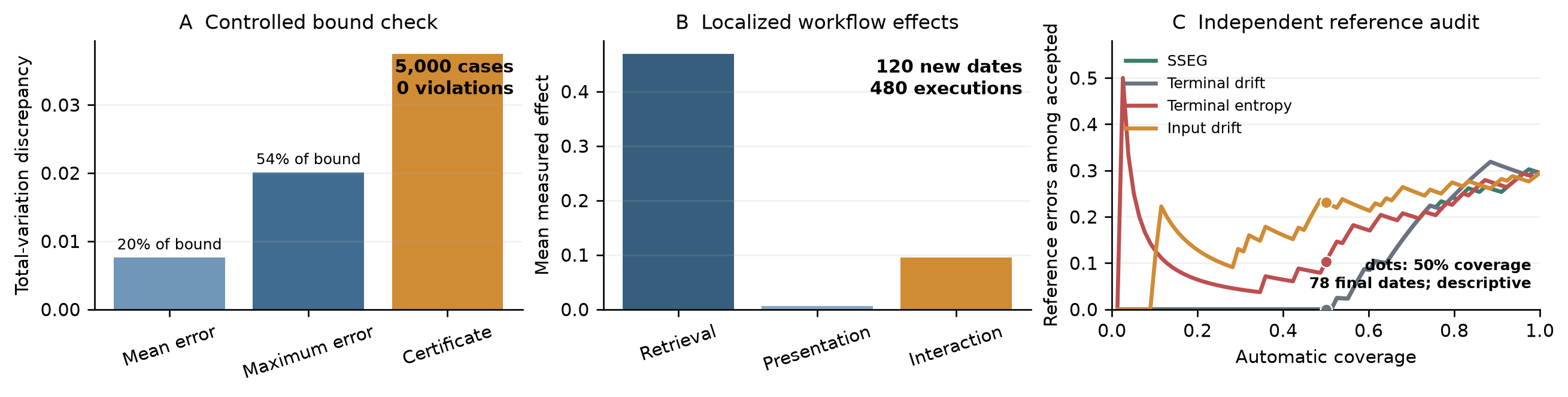}
\caption{Evidence ladder.  A: mean and maximum controlled discrepancies use 20\% and 54\% of the valid pathwise bound.  B: the 120-date requalification set separates retrieval, presentation and interaction effects.  C: a post-confirmatory descriptive audit uses a prespecified market-state reference distinct from the terminal-discrepancy endpoint.  On the final 78 dates, SSEG and equal-information direct terminal drift both have zero reference errors at 50\% coverage; SSEG retains the retrieval and presentation decomposition that terminal drift discards.  The reference is an arithmetic benchmark derived from dated evidence, not a human judgment or future market outcome.}
\label{fig:certificate-rag}
\end{figure*}

\paragraph{Finding 2: the pathwise bound holds in the controlled test.}
The controlled experiment has zero bound violations: terminal mean and maximum total-variation error are .00759 and .02009, below the .03750 certificate.  The mean slack factor is 4.94, so the uniform certificate is conservative but valid.  A 35,000-case misspecification stress test then gives the robust envelope 100\% coverage and zero false passes in every scenario, while the nominal false-pass rate reaches 42.2\% under combined stress.  With missing telemetry, robust automatic-action coverage falls to 4.9\% and review cost raises loss from .0802 to .0859.  These cases are deliberately controlled because certificate coverage cannot be observed without knowing the graph, local defects and reference law; they test the theorem under known truth and are not evidence about the frequency of organic production faults.  Cost-sensitive topology experiments appear later in the appendices.  \textbf{Takeaway: graph uncertainty must widen the certificate or trigger escalation; safety may reduce availability.}

\paragraph{Finding 3: the graph identifies which workflow component moved.}
On the 120-date requalification set, the crossed design separates mean retrieval, presentation and interaction effects of .469, .006 and .095.  Retrieval-edge and terminal discrepancies have $r=.99997$, and every discrepancy respects its pathwise bound.  These quantities derive from the same crossed phrase laws, so their near-perfect association is a structural consistency check, not independent predictive evidence.  In particular, a terminal-drift score perfectly identifies a label defined by that same drift; terminal entropy is not an adequate substitute comparator for such a claim.

We therefore retain the original frozen discrepancy analysis as a qualification audit, not evidence of superiority over terminal monitoring, and add a clearly labelled post-confirmatory check against the distinct reference state prespecified in the design.  Among 78 final dates, 23 mean four-cell decisions differ from that reference.  SSEG ranks these errors with AUROC .848, compared with .827 for equal-information terminal drift, .752 for terminal entropy and .649 for generic input drift.  SSEG and direct terminal drift both have zero reference errors at 50\% coverage, and their selective-risk difference is not significant under an eight-date moving-block bootstrap.  \textbf{Takeaway: SSEG matches direct terminal monitoring here while additionally identifying whether retrieval, presentation or their interaction generated the discrepancy; it does not establish better categorical prediction or market correctness.}

\paragraph{Organic source-order limitation.}
A post-run audit found that an earlier DRAGged into CONFLICTS \cite{cattan2025conflicts} pilot scored only the first token of multi-token natural-language labels; one model also exposed the reasoning rather than final-answer channel.  A subsequent audit found candidate-dependent left padding in the complete-sequence scorer, which could shift the shared prefix across alternatives.  All affected measurements are therefore excluded.  The corrected protocol right-pads candidates, verifies identical prefix positions and expressed mass, scores only the final-answer channel and must pass model-specific sentinels before a full run.  No CONFLICTS result enters the evidence ladder until that cross-model requalification completes.

\paragraph{Matched Self-RAG path-qualification study.}
The reviewer-motivated comparison asks whether retained graph structure adds governance information inside a fixed self-reflective generator rather than against a weaker generator.  We pin the released Self-RAG model and construct an explicit two-stage evidence beam for HotpotQA bridge questions \cite{selfrag2024,hotpot2018}.  Each trace retains five first-hop branches; the best two are expanded through four distinct second-hop passages, yielding 13 nodes and 11 root-to-leaf candidate paths.  A path is labelled safe only when its set of passage titles equals the two independently annotated supporting-document titles.  Answers and support titles are firewalled from feature construction.  Eight sentinels and all 384 evaluation traces pass the prespecified node-count, depth, identity and label-leakage audit.

The flat comparator is class-balanced logistic regression on terminal Self-RAG final, groundedness, relevance and accumulated scores.  SSEG receives those identical terminal measurements plus first-hop scores, minimum path groundedness and relevance, ordered parent--child score change, sibling and first-hop rank, and distinct-passage identity.  Five-fold question-grouped cross-validation on 288 previously opened questions fixes the feature set and regularization.  The evaluation population is then selected by a salted hash from eligible official HotpotQA distractor-development questions after excluding every development identifier.  It contains 384 questions, 4,224 paths and 595 complete support paths; every path from a question remains together in each of 5,000 paired bootstrap resamples.

SSEG raises path-level AUROC from .783 to .827 (paired interval $[.032,.056]$) and average precision from .331 to .392 ($[.027,.095]$).  At a frozen 10\% shortlist, precision rises from 36.3\% to 43.6\% ($[2.1,13.5]$ percentage points) and recall from 25.7\% to 30.9\% ($[1.5,9.6]$ points).  This is about a 20\% relative increase in complete support paths per reviewed candidate.  However, rank-first accuracy declines from 52.8\% to 47.5\% (difference interval $[-10.4,-0.5]$ points), and mean first-safe-path rank worsens from 2.18 to 2.39.  A selected two-hop path touches two of the 13 retained evidence nodes, so either ranker can localize subsequent inspection to 15.4\% of the trace; the graph's demonstrated advantage is selecting a better \emph{set} of paths within that budget.  It does not improve generation, answer accuracy or definitive path choice.

An earlier opened pilot is retained as a limitation rather than pooled with the confirmatory result.  On one-hop ASQA, SSEG did not improve answer-level selective error; on the first HotpotQA beam, answer-recovery AURC also failed to improve.  These failures motivated the estimand used above: a path graph should be evaluated against independently labelled complete evidence paths, while final-answer correctness remains a separate endpoint.  The frozen protocol, coefficients, generated traces, full audit and analysis are stored under \texttt{selfrag\_sseg\_path\_governance\_v3}.

\paragraph{Three/four-hop Self-RAG source-sensitivity study.}
We separately test whether retained path measurements identify evidence changes that materially alter a self-reflective model's response.  The released Self-RAG Llama-2-7B checkpoint is pinned at revision \texttt{1902613}.  Official MuSiQue train questions supply 80 development questions and the untouched development split supplies 160 evaluation questions, balanced between three and four hops.  For each question, the independently annotated supporting chain is paired with one condition per hop in which that source is replaced by the closest-length non-supporting paragraph from the same question.  Opaque condition identifiers and model inputs contain neither the fault index nor answer label.  This yields 280 development and 560 evaluation source-intervention pairs.  Eighteen sentinel conditions precede scaling; all 1,098 generated conditions preserve exact source order, depth and node identity, and all scores are finite and non-collapsed.

The terminal comparator receives 13 features: depth and clean, substituted and changed terminal final, groundedness, relevance and accumulated scores.  SSEG receives that exact vector as a prefix plus 32 frozen path features describing ordered level changes, absolute changes, bottlenecks and downstream accumulation.  Thus the comparison isolates retained path information rather than additional terminal telemetry.  Response distance is one minus normalized-token Jaccard similarity between matched clean and substituted outputs; its material threshold is the development 75th percentile (.5669).  On 560 untouched pairs, terminal versus SSEG average precision is .283 versus .538, AUROC .562 versus .766 and precision in the highest-risk quartile 25.7\% versus 53.6\%.  The 5,000-resample question-clustered AP difference is .254 with interval $[.162,.340]$.  A label-blind permutation audit gives mean null AP .249 and one-sided $p<.0005$.  At a strong-grounding threshold fixed by the development clean-node upper quartile, 19.8\% of substituted unsupported nodes still pass.

The exact-hop endpoint is null: local Self-RAG reaches 38.6\% and SSEG 38.4\%, with difference interval $[-3.5,3.4]$ points; within-one accuracy is significantly worse for SSEG.  The preregistered success rule therefore passes through response-sensitivity AP, not localization.  A post hoc positional audit shows why the two estimands diverge.  For three-hop chains, SSEG selects the first hop 164 times and the second only 22 times across 240 balanced interventions; corresponding counts are 122 and 52 across 320 four-hop interventions.  Yet second-hop substitutions have the largest mean response distance (.516 for three-hop and .506 for four-hop) and are localized in only 15.0\% and 20.0\% of cases.  The absolute-score localizer observes a single substituted trace, while the successful sensitivity model uses paired clean-versus-substituted changes across the complete path.  In a linear chain, a source perturbation may be absorbed locally and become visible only downstream, so maximal reflection defect and intervention location need not coincide.

\paragraph{Post-hoc paired edge-contrast pilot.}
The failure above isolates a measurement problem, not an absence of graph identity.  A traced DAG already records which source occupies each evidence edge; what was missing from the absolute-score localizer was a local defect measurement analogous to $b_v$ in Theorem~\ref{thm:graph}.  Before fitting a replacement method, we froze a diagnostic protocol using the original 80-question development split and 160-question evaluation split.  The paired-node comparator receives depth, normalized hop and clean-versus-substituted changes in final, groundedness, relevance and accumulated scores.  Paired-edge SSEG receives this exact ten-feature vector as a prefix plus parent and successor changes, incoming and outgoing edge contrasts, suffix propagation and normalized path contribution.  Neither model receives passage text, passage identity or the known intervention index.

Across the 560 evaluation interventions, absolute local Self-RAG localizes 38.6\%, paired-node contrasts localize 75.2\% and paired-edge SSEG localizes 88.4\%; the paired-edge gain over paired-node measurement is 13.2 points with a 5,000-resample question-clustered interval $[10.1,16.3]$.  Paired-edge within-one accuracy is 99.1\% and its mean reviewed fraction is .674, compared with .709 for paired node and .845 for absolute local scoring.  Development and evaluation identifiers are disjoint, every feature is finite, the paired-node vector is an exact prefix of the edge vector and a 2,000-permutation randomized-position audit gives $p<.0005$.  Thus the retained graph adds attribution information after the local contrast is measured.  Because this method was designed after inspecting the failed endpoint, these results are hypothesis-generating and are not used to revise the confirmatory claim.

\paragraph{Stateful tool-use limitation.}
ToolSandbox supplies executable tools, persistent state, on-policy dialogue and a milestone DAG over intermediate and final states \cite{toolsandbox2025}.  We freeze 48 executions from 12 multi-turn state-dependent scenario families and cross the base task with ten distractor tools, scrambled tool names and scrambled argument descriptions; the identical suite is run with \texttt{gpt-4.1-mini-2025-04-14} and \texttt{gpt-4o-mini-2024-07-18}, yielding 96 executions.  Forty contain at least one failed intermediate milestone.  A label-blind monitor using missing expected calls and order violations completely covers 75.0\% of faulty paths while routing 61.0\% of intermediate nodes to review.  The misses commonly invoke the expected tool with schema-valid but semantically wrong arguments, so tool presence does not establish tool correctness.  Once ToolSandbox's independently declared argument/state validator supplies a local contract alarm, SSEG routes all 40 failed paths while reviewing 71.3\% of intermediate nodes, 28.8\% fewer than exhaustive inspection.  This is post-alarm routing, not predictive detection: trace-only SSEG ties the flat contract monitor in both arms.  The study uses 503 agent calls, 236 simulated-user calls and 566.6 seconds of summed scenario wall time; deterministic offline graph scoring adds no model calls and takes under 0.5 milliseconds per case on the test machine.

\paragraph{Exploratory human-rated provenance analysis.}
We use 2,886 corrected human-rated ALCE claim--citation links \cite{alce2023}, split by question into 1,419 development and 1,467 untouched evaluation links.  Citation markers are removed from scored targets, control passages are token-length matched and repeated citations to the same sentence receive equal total weight.  Each cited passage is replaced while question, answer and scorer remain fixed.  SSEG records the resulting teacher-forced sentence-probability change; human judgments serve only as evaluation labels.

The graph also supplies a diagnostic decomposition.  An audit verified every target, control identifier and human label and replaced the earlier model-specific matching with one frozen common-control map.  The same source replacement is therefore used for every displayed model, although tokenizer-relative lengths necessarily differ.  Controls are selected without support labels and remain within the frozen development/evaluation partition; they do not import target outcomes.  The resulting five-model comparison holds target and control text fixed and remains a descriptive map of aligned-word uncertainty and source influence rather than causal attribution or field prevalence.  The prior Mistral output is excluded until its duplicated-BOS rerun completes.

\textbf{Where uncertainty enters.}  SSEG does not reduce provenance to one confidence score.  The source--claim edge retains the distribution of word-level probability changes under source substitution: its mean records overall sensitivity, its upper tail records locally strong dependence and its positive-word fraction records how broadly support extends across the claim.  Source-conditioned surprisal records residual language uncertainty, while disagreement with an entailment model or lexical overlap identifies a channel conflict.  On the held-out ALCE links, this decomposition raises rank-1 localization within affected questions from 58.7\% for the HHEM-plus-lexical flat comparator to 68.0\%, while global AUROC remains statistically tied.  The practical gain is therefore a source-specific diagnostic and revalidation action, not a claim of uniformly better screening.

\begin{table*}[t]
\centering
\footnotesize
\setlength{\tabcolsep}{4pt}
\caption{The main evidence ladder.  Each row answers a different part of the same governance claim; heterogeneous endpoints are not pooled.}
\begin{tabular}{>{\raggedright\arraybackslash}p{.17\textwidth}>{\raggedright\arraybackslash}p{.18\textwidth}>{\raggedright\arraybackslash}p{.40\textwidth}>{\raggedright\arraybackslash}p{.17\textwidth}}
\toprule
Question & Evaluation & Main result & Permitted conclusion\\
\midrule
Propagation & 35,000 known-truth stress cases & Robust envelope attained complete observed coverage and zero false passes under the tested misspecifications. & Conditional theorem verification, not field prevalence.\\
Workflow diagnosis & 78 final financial dates; 312 executions & SSEG and direct terminal drift both admit zero reference errors at 50\% coverage; SSEG retains typed path attribution. & Post-confirmatory diagnosis, not superior prediction or market correctness.\\
Conflict-path localization & 960 untouched MAGIC cases; 8,592 candidate paths & Frozen graph--lexical ranking raises rank-1 localization from 75.8\% to 80.3--82.3\% across five model families. & Independently labelled localization and reduced ranked search.\\
\bottomrule
\end{tabular}
\label{tab:transfer}
\end{table*}

\section{Additional claim limitations}
\label{app:limitations}

The purpose of this section is to delimit the statistical, causal and operational claims.  The 35,000 known-truth cases verify a conditional theorem, not production prevalence.  The financial study supports decomposition and bounds, not superior prediction; its independent-reference comparison is post-confirmatory.  ALCE supports source-reliance diagnosis, not causal attribution or an architecture ranking.  RAGTruth, software, retrieval-repair and ToolSandbox studies test transfer under their stated controls, not universal diagnosis or improvement.  Live traces may omit edges, tested models do not establish frontier-model transport, and interventions add unoptimized latency.  Semantic calibration and identification remain separate requirements \cite{dixonsemantic2026}.  \textbf{Stability is not correctness, and controlled sensitivity is not causal attribution without identified interventions.}

\renewcommand{\theHfigure}{A.\arabic{figure}}
\renewcommand{\theHtable}{A.\arabic{table}}
\section{Condensed secondary evidence and audit limitations}
\label{app:secondary}

The purpose of this section is to summarize transfer and failure-limitation evidence without enlarging the paper's confirmatory claim.  Table~\ref{tab:secondary-summary} states each population, result and limitation; frozen records and scripts remain in the anonymous archive.

\begin{table}[H]
\centering
\scriptsize
\setlength{\tabcolsep}{3pt}
\caption{Secondary evidence, retained for auditability but not pooled with the headline ladder.}
\label{tab:secondary-summary}
\begin{tabular}{p{.18\linewidth}p{.22\linewidth}p{.50\linewidth}}
\toprule
Study & Population & Result and permitted interpretation\\
\midrule
RCAEval & 375 untouched incidents & 82.1\% top-one localization; tied with equal-telemetry flat rules, so no superiority claim.\\
SWE-rebench & 1,614 issues; 216 unseen repositories & At 25\% coverage, error falls from 4.0\% to 2.4\%; the interval touches zero.\\
Retrieval repair & 7,500 controlled two-fault pipelines & Targeted repair restores affected descendants under injected faults; organic diagnosis is untested.\\
ToolSandbox & 96 executions; two agent models & Retained paths reduce post-alarm review by 28.8\%; trace-only SSEG does not beat the flat alarm.\\
Self-RAG & 384 questions; 560 source interventions & Paths improve shortlisting and sensitivity detection; exact-hop localization needs paired edge contrasts.\\
RAGTruth and transfer & 2,000 responses plus two holdouts & Modest word-level gain; summary-level and incomplete-law transfer gates fail.\\
\bottomrule
\end{tabular}
\end{table}

\subsection{Continuation laws, provenance and ALCE protocol}
\label{app:lexical}
\label{app:provenance}
\label{app:alce}
Candidate phrases are teacher-forced under one recorded prefix; residual mass is retained and low-mass cases abstain.  Runtime provenance records source, span, time, hash and downstream transfers, but does not infer training-data attribution.  Controlled substitutions hold the question, answer and scorer fixed; without identification, dependence is diagnostic rather than causal.  ALCE's 2,886 human-rated links are split by question into 1,419 development and 1,467 evaluation links.  Scored targets omit citation markers; label-blind controls are token-length matched; repeated citations share weight; the control map is frozen across five models; and intervals cluster by question.  The result measures and routes source reliance, not the truth of a human-supported citation.

\section{Observation profiles and claim limitations}
\label{app:observation-profiles}

The purpose of this section is to distinguish what can be observed under black-box, hosted open-weight and self-hosted instrumentation profiles, and what may validly be inferred from each.

Frontier black-box services expose generated text and sometimes truncated top-$k$ alternatives; missing mass must remain explicit.  Hosted open-weight endpoints may expose candidate scores without internal activations.  Self-hosting permits full logits, hidden-state and routing diagnostics.  These extra observables refine the graph but do not become causal explanations without identified interventions.  No experiment claims access to hidden beliefs, document-level pretraining attribution or transport beyond the frozen graphs, models and intervention families.

\begin{table}[H]
\centering
\small
\caption{Capability-aware observation profiles.}
\label{tab:observation-profiles}
\begin{tabular}{p{.21\linewidth}p{.28\linewidth}p{.42\linewidth}}
\toprule
Profile & Observable measurement & Claim limitation\\
\midrule
Frontier black box & Generated text, returned top-$k$ alternatives and residual mass & Partial phrase law unless every declared alternative is exposed; no internal attribution.\\
Hosted open weight & Pinned endpoint and candidate or token log probabilities & Complete candidate law only when every complete phrase is scored under the same declared prefix.\\
Self-hosted instrumented & Full logits and candidate laws; optional hidden-state and routing summaries & Deeper diagnostics are available, but internal signals are not causal explanations without identification.\\
\bottomrule
\end{tabular}
\end{table}

The audit record includes graph topology, component and model revisions, tokenizer and termination rule, prompts, phrase family, retrieved-document identifiers and vintages, metric, thresholds, split hashes and unavailable channels.  Runtime source provenance and model-training provenance are distinct.  The former may be verified through source identifiers, accessed spans, timestamps and content hashes.  The latter is only partially observed when the base training corpus is undisclosed; open weights do not retroactively identify document-level training attribution.

\ifdefined\supplementonly

\fi
\fi


\begin{thebibliography}{99}
\bibitem{blackwell1953} D. Blackwell. Equivalent comparisons of experiments. \emph{Annals of Mathematical Statistics}, 24(2):265--272, 1953.
\bibitem{frechet1951} M. Fr\'echet. Sur les tableaux de corr\'elation dont les marges sont donn\'ees. \emph{Annales de l'Universit\'e de Lyon, Section A}, 14:53--77, 1951.
\bibitem{csiszar1967} I. Csisz\'ar. Information-type measures of difference of probability distributions and indirect observations. \emph{Studia Scientiarum Mathematicarum Hungarica}, 2:299--318, 1967.
\bibitem{dobrushin1956} R. L. Dobrushin. Central limit theorem for nonstationary Markov chains. I. \emph{Theory of Probability and Its Applications}, 1(1):65--80, 1956; and II, 1(4):329--383, 1956.
\bibitem{makur2020} A. Makur and L. Zheng. Comparison of contraction coefficients for $f$-divergences. \emph{Problems of Information Transmission}, 56:103--156, 2020.
\bibitem{ovadia2019} Y. Ovadia et al. Can you trust your model's uncertainty? Evaluating predictive uncertainty under dataset shift. \emph{NeurIPS}, 2019.
\bibitem{geifman2019} Y. Geifman and R. El-Yaniv. SelectiveNet: A deep neural network with an integrated reject option. \emph{Proceedings of ICML}, 97:2151--2159, 2019.
\bibitem{angelopoulos2024} A. N. Angelopoulos, S. Bates, A. Fisch, L. Lei, and T. Schuster. Conformal risk control. \emph{ICLR}, 2024.
\bibitem{kadavath2022} S. Kadavath et al. Language models (mostly) know what they know. arXiv:2207.05221, 2022.
\bibitem{kuhn2023} L. Kuhn, Y. Gal, and S. Farquhar. Semantic uncertainty: linguistic invariances for uncertainty estimation in natural language generation. \emph{ICLR}, 2023.
\bibitem{heo2025} J. Heo, M. Xiong, C. Heinze-Deml, and J. Narain. Do LLMs estimate uncertainty well in instruction-following? \emph{ICLR}, 2025.
\bibitem{liu2025} G. K.-M. Liu, G. Yona, A. Caciularu, I. Szpektor, T. G. J. Rudner, and A. Cohan. MetaFaith: Faithful natural language uncertainty expression in LLMs. \emph{Proceedings of EMNLP}, pages 29600--29644, 2025.
\bibitem{kalai2025} A. T. Kalai, O. Nachum, S. S. Vempala, and E. Zhang. Why language models hallucinate. arXiv:2509.04664, 2025.
\bibitem{fu2025} Y. Fu, X. Wang, H. Zhang, Y. Tian, and J. Zhao. Deep Think with Confidence. \emph{International Conference on Learning Representations}, 2026.
\bibitem{nakkiran2026} P. Nakkiran et al. Trained on tokens, calibrated on concepts: The emergence of semantic calibration in LLMs. \emph{ICLR}, 2026.
\bibitem{ye2026} Z. Ye et al. Uncertainty quantification for LLM function-calling. arXiv:2604.22985, 2026.
\bibitem{lewis2020} P. Lewis et al. Retrieval-augmented generation for knowledge-intensive NLP tasks. \emph{NeurIPS}, 2020.
\bibitem{paglieri2025} D. Paglieri et al. BALROG: Benchmarking agentic LLM and VLM reasoning on games. \emph{ICLR}, 2025.
\bibitem{toolsandbox2025} J. Lu, T. Holleis, Y. Zhang, B. Aumayer, F. Nan, H. Bai, S. Ma, S. Ma, M. Li, G. Yin, et al. ToolSandbox: A stateful, conversational, interactive evaluation benchmark for LLM tool use capabilities. \emph{Findings of NAACL}, pages 1160--1183, 2025.
\bibitem{provagent2025} R. Souza, A. Gueroudji, S. DeWitt, D. Rosendo, T. Ghosal, R. Ross, P. Balaprakash, and R. Ferreira da Silva. PROV-AGENT: Unified provenance for tracking AI agent interactions in agentic workflows. \emph{2025 IEEE International Conference on eScience}, pages 467--473, 2025.
\bibitem{raji2020} I. D. Raji, A. Smart, R. N. White, M. Mitchell, T. Gebru, B. Hutchinson, J. Smith-Loud, D. Theron, and P. Barnes. Closing the AI accountability gap: Defining an end-to-end framework for internal algorithmic auditing. \emph{Proceedings of the 2020 Conference on Fairness, Accountability, and Transparency}, pages 33--44, 2020.
\bibitem{mitchell2019} M. Mitchell, S. Wu, A. Zaldivar, P. Barnes, L. Vasserman, B. Hutchinson, E. Spitzer, I. D. Raji, and T. Gebru. Model cards for model reporting. \emph{Proceedings of the Conference on Fairness, Accountability, and Transparency}, pages 220--229, 2019.
\bibitem{arnold2019} M. Arnold, R. K. E. Bellamy, M. Hind, S. Houde, S. Mehta, A. Mojsilovi\'c, R. Nair, K. N. Ramamurthy, A. Olteanu, D. Piorkowski, D. Reimer, J. Richards, J. Tsay, and K. R. Varshney. FactSheets: Increasing trust in AI services through supplier's declarations of conformity. \emph{IBM Journal of Research and Development}, 63(4/5):6:1--6:13, 2019.
\bibitem{nistairmf2023} E. Tabassi. Artificial Intelligence Risk Management Framework (AI RMF 1.0). NIST AI 100-1, National Institute of Standards and Technology, 2023.
\bibitem{arcjsd2025} R. Li et al. Attributing response to context: A Jensen--Shannon divergence driven mechanistic study of context attribution in retrieval-augmented generation. \emph{International Conference on Learning Representations}, 2026.
\bibitem{sgic2025} G. Chen et al. SGIC: A self-guided iterative calibration framework for RAG. \emph{Proceedings of the 63rd Annual Meeting of the Association for Computational Linguistics (Volume 1: Long Papers)}, pages 28357--28370, 2025.
\bibitem{bergen2024} D. Rau et al. BERGEN: A benchmarking library for retrieval-augmented generation. \emph{Findings of EMNLP}, 2024.
\bibitem{ares2024} J. Saad-Falcon et al. ARES: An automated evaluation framework for retrieval-augmented generation systems. \emph{NAACL}, 2024.
\bibitem{ragas2024} S. Es, J. James, L. Espinosa Anke, and S. Schockaert. RAGAs: Automated evaluation of retrieval augmented generation. \emph{Proceedings of EACL: System Demonstrations}, pages 150--158, 2024.
\bibitem{factscore2023} S. Min et al. FActScore: Fine-grained atomic evaluation of factual precision in long-form text generation. \emph{Proceedings of EMNLP}, pages 12076--12100, 2023.
\bibitem{selfrag2024} A. Asai et al. Self-RAG: Learning to retrieve, generate, and critique through self-reflection. \emph{ICLR}, 2024.
\bibitem{halueval2023} J. Li et al. HaluEval: A large-scale hallucination evaluation benchmark for large language models. \emph{Proceedings of EMNLP}, pages 6449--6464, 2023.
\bibitem{swerebench2025} I. Badertdinov, A. Golubev, M. Nekrashevich, A. Shevtsov, S. Karasik, A. Andriushchenko, M. Trofimova, D. Litvintseva, and B. Yangel. SWE-rebench: An automated pipeline for task collection and decontaminated evaluation of software engineering agents. arXiv:2505.20411, 2025.
\bibitem{scifact2020} D. Wadden et al. Fact or fiction: Verifying scientific claims. \emph{EMNLP}, 2020.
\bibitem{hotpot2018} Z. Yang et al. HotpotQA: A dataset for diverse, explainable multi-hop question answering. \emph{EMNLP}, 2018.
\bibitem{musique2022} H. Trivedi, N. Balasubramanian, T. Khot, and A. Sabharwal. MuSiQue: Multihop questions via single-hop question composition. \emph{Transactions of the Association for Computational Linguistics}, 10:539--554, 2022.
\bibitem{pham2025rcaeval} L. Pham, H. Zhang, H. Ha, F. Salim, and X. Zhang. RCAEval: A benchmark for root cause analysis of microservice systems with telemetry data. \emph{Companion Proceedings of the ACM Web Conference}, pages 777--780, 2025.
\bibitem{alce2023} T. Gao, H. Yen, J. Yu, and D. Chen. Enabling large language models to generate text with citations. \emph{Proceedings of the 2023 Conference on Empirical Methods in Natural Language Processing}, pages 6465--6488, 2023.
\bibitem{vitaminc2021} T. Schuster, A. Fisch, and R. Barzilay. Get your Vitamin C! Robust fact verification with contrastive evidence. \emph{Proceedings of NAACL}, pages 624--643, 2021.
\bibitem{ragtruth2024} C. Niu et al. RAGTruth: A hallucination corpus for developing trustworthy retrieval-augmented language models. \emph{Proceedings of ACL}, pages 10862--10878, 2024.
\bibitem{hhem2024} Vectara. HHEM-2.1-Open hallucination evaluation model. Hugging Face model artifact, DOI:10.57967/hf/3240, 2024.
\bibitem{qwen3} A. Yang et al. Qwen3 technical report. arXiv:2505.09388, 2025.
\bibitem{mistral7b} A. Q. Jiang et al. Mistral 7B. arXiv:2310.06825, 2023.
\bibitem{lee2025magic} J. Lee, K. Lee, and T. Kim. MAGIC: A multi-hop and graph-based benchmark for inter-context conflicts in retrieval-augmented generation. \emph{Findings of EMNLP}, 2025.
\bibitem{cattan2025conflicts} A. Cattan et al. DRAGged into CONFLICTS: Detecting and addressing conflicting sources in search-augmented LLMs. arXiv:2506.08500, 2025.
\bibitem{dixonsemantic2026} M. Dixon. Calibrating semantic uncertainty from observable language-model probabilities. arXiv:2607.17447, 2026.
\bibitem{fingovbench2026} M. F. Dixon. FinGovBench: A closed-loop benchmark for financial AI governance. BeliefLens release candidate 1.0.0-rc4, 2026. \url{https://huggingface.co/datasets/BeliefLens/FinGovBench}.
\end{thebibliography}
\end{document}